\documentclass[preprint,12pt]{elsarticle}
\usepackage{graphicx}
\usepackage{booktabs}
\usepackage{amsmath}
\usepackage{amssymb}
\usepackage{nicematrix,tikz}
\usepackage{algorithm}
\usepackage{algpseudocode}
\usepackage{changepage}
\usepackage{hyperref}
\usepackage{amsthm}
\usepackage{threeparttable}
\usepackage{adjustbox}
\usepackage{array}
\usepackage{breqn}
\usepackage{float}
\usepackage{caption}
\usepackage{titlesec}
\newtheorem{theorem}{Theorem}
\newtheorem*{theorem*}{Theorem}
\newtheorem{lemma}{Lemma}
\newtheorem{proposition}{Proposition}

\usepackage{subcaption}
\usepackage{bm}

\usepackage{multirow}

\usetikzlibrary{quantikz2}
\journal{Engineering Applications of Artificial Intelligence}

\begin{document}

\begin{frontmatter}




\title{Mitigating Measurement-Induced Training Instability in Hybrid Quantum Neural Networks for Protein Classification}
\author[1]{Milton Mondal}
\author[1]{Sushovan Chanda}
\author[1]{Mohamad Mahdi Alawieh}
\author[1]{Brijesh Sukhadiya}
\author[1]{Donatus Krah}
\author[1]{Clinton Gonsalves}
\author[1]{Antonios Ntolkeras}
\author[1]{Silvio O. Rizzoli\corref{cor}}
\ead{silvio.rizzoli@med.uni-goettingen.de}
\author[1]{Ali H. Shaib\corref{cor}}
\ead{ali.shaib@med.uni-goettingen.de}
\cortext[cor]{Corresponding Authors}

\address[1]{Department of Neuro- and Sensory Physiology, University Medical Center Göttingen, Göttingen, Germany
} 

\begin{abstract}

Hybrid Quantum Neural Network (QNN) classifiers produce logits as expectation values of quantum measurement operators. For standard Pauli measurements, these outputs are intrinsically bounded to the interval $[-1,1]$. When such bounded logits are used directly with the cross-entropy loss applied
to softmax-normalized logits for multi-class classification, the loss function operates in a regime of weak sensitivity to logit differences. As a consequence, parameter gradients are suppressed,
leading to unstable optimization in variational quantum classifiers (VQCs). In this work, we identify this effect as \emph{measurement-induced logit contraction}, a previously uncharacterized source of trainability degradation in hybrid QNNs. To address this limitation, we introduce a learnable scaling parameter, termed \emph{Quantum Measurement Temperature (QMT)}, which rescales quantum measurement outputs prior to the loss.
Unlike post-hoc calibration, \emph{QMT acts during training and compensates for the physically imposed bounds on quantum measurement outputs}. This rescaling increases gradient magnitude and variance, thereby improving loss sensitivity. The proposed mechanism is architecture-agnostic and does not modify the quantum ansatz, circuit depth, or measurement operators. Experiments on fluorescence microscopy images and a six-class variant of Fashion MNIST demonstrate that QMT consistently enhances logit separation, strengthens gradients, stabilizes training across random initializations, and improves classification accuracy, relative to unscaled measurement readouts. These results demonstrate that QMT enables stable and reliable training of hybrid QNNs for practical applications.

\end{abstract}






\end{frontmatter}



\section{Introduction}
Hybrid quantum neural networks combine variational quantum circuits with classical optimization and have been explored for classification tasks mainly as feasibility studies rather than as robust learning systems~\cite{biamonte2017quantum,benedetti2019parameterized,abbas2021power}. Despite progress in circuit design, data encoding, and optimization strategies, training instability remains a major obstacle in utilizing quantum neural networks (QNNs) for practical applications. In particular, beyond circuit-level effects, hybrid QNNs suffer from a previously overlooked bottleneck at the measurement-loss interface, where bounded quantum measurement outputs are directly used as logits in softmax-based cross entropy losses. 

Prior work has shown that random parameter initialization, global cost functions, high expressivity, and entanglement growth can lead to vanishing gradients through concentration of measure, giving rise to barren plateaus~\cite{mcclean2018barren} and other challenges in optimizing quantum models~\cite{cerezo2021cost}. These studies characterize how the model architecture can influence parameter gradients~\cite{donaire2026hybrid}. However, this line of work implicitly assumes that the gradients computed for the model parameters are propagated properly through the classical loss function without introducing distortion~\cite{cerezo2021cost, hubregtsen2021evaluation}. Although this assumption is reasonable for fully classical models, it does not hold for hybrid quantum classifiers. 

In quantum models, predictions are derived from the expectation values of quantum measurement operators, which are intrinsically bounded to the interval $[-1,1]$ for standard Pauli measurements. Despite this fundamental constraint, hybrid QNNs typically adopt the same softmax-based probability assignment and cross-entropy loss function used in classical deep learning, where logits are unbounded and can dynamically rescale during training~\cite{goodfellow2016deep,guo2017calibration}. This creates a structural mismatch between bounded quantum measurement outputs and a loss function whose sensitivity depends critically on logit magnitude. 

When quantum measurement outputs are confined to a narrow numerical range, gradients with respect to quantum parameters are suppressed. The resulting failure mode is distinct from previously studied circuit level gradient collapse mechanisms, which have been discussed in the past ~\cite{mcclean2018barren, holmes2022connecting,  ortiz2021entanglement} and arises from post measurement processing. To address this limitation, we introduce Quantum Measurement Temperature (QMT), a learnable scalar applied to quantum measurement outputs. By rescaling the expectation values before softmax computation, QMT restores an effective logit dynamic range and shifts the loss to regions with higher gradient sensitivity. 

The main contributions of this work are:
\begin{itemize}
\item Identification of a measurement-induced logit compression mechanism in hybrid quantum classifiers, showing how bounded quantum measurements suppress class separability and gradient flow in multi-class settings.

\item	Introduction of Quantum Measurement Temperature as an architecture-agnostic and theoretically grounded method for restoring effective logit range and stabilizing training, without modifying the quantum circuit.

\item	Analytical characterization of how temperature scaling affects achievable confidence, loss sensitivity, gradient magnitude, and gradient variance.

\item	Experimental validation microscopy image classification and standard vision benchmarks, demonstrating improved training stability and classification performance across diverse architecture configurations.
\end{itemize}


\label{sec1}

\section{Related Work and Motivation}

The trainability of variational quantum algorithms (VQAs) and quantum neural networks (QNNs) is determined primarily by the strengths of gradients during optimization. Several previous works have established that gradient magnitudes in parameterized quantum circuits can vanish
exponentially under various conditions~\cite{heyraud2023efficient, grant2019initialization}. These include the discovery of barren plateau (BP) problem by McClean \textit{et al.}~\cite{mcclean2018barren}, which showed that
for highly expressive or random ans\"atze, the variance of gradients decays
exponentially with the number of qubits. Subsequent studies identified
different sources of vanishing gradients while training a QNN, including dependence on global versus local cost
functions~\cite{cerezo2021cost}, excessive expressivity~\cite{sim2019expressibility} and entanglement
growth~\cite{holmes2022connecting,ortiz2021entanglement}, noise driven
gradient collapse~\cite{wang2021noise}, and the
algebraic structure of the circuit generators~\cite{ragone2024lie}.

Although previous papers characterize gradient concentration from ans\"atze point of view, recent research has shown that data structure~\cite{thanasilp2023subtleties}, encoding schemes~\cite{schuld2021effect,berberich2024training},
and task-dependent observables also influence QNN trainability~\cite{kashif2024nrqnn}. For example, correlated parameterizations can increase gradient magnitudes
relative to uncorrelated initializations~\cite{volkoff2021large}. 
Dissipative QNNs exhibit clear regimes where gradients either remain 
informative or collapse~\cite{sharma2022trainability}.
Dataset structure and encoding choices can also influence gradient
statistics, leading to ``dataset induced'' barren plateaus
even for shallow circuits~\cite{thanasilp2023subtleties}. Whereas modifying architectures such as QCNNs~\cite{cong2019quantum, mahendra2026quantum},
tensor network inspired designs~\cite{huggins2019towards,rieser2023tensor},
problem dependent ans\"atze~\cite{anschuetz2022quantum},
and quantum kernel based methods~\cite{schuld2021supervised,huang2021power}
aim to maintain non vanishing gradients by controlling locality,
expressivity, or parameterization geometry~\cite{lozano2025practical}. 

Despite these advances, existing theory predominantly focuses on gradients
generated \emph{within} the quantum circuit~\cite{ anschuetz2022quantum, caro2022generalization, pesah2021absence}. However, when QNNs are used for
classification, the model output is not a scalar observable but a
vector of expectation values that serve as logits for a classical loss
function. For standard Pauli measurements, each logit is confined to a narrow interval, introducing a physically imposed bound that has no analogue
in classical neural networks. Prior trainability analyses implicitly assume
that the loss function faithfully propagates circuit-level gradients, without
accounting for how bounded quantum outputs interact with the sensitivity of
the loss itself~\cite{cerezo2021cost, holmes2022connecting, long2025hybrid}. 
Softmax-based cross-entropy losses are highly sensitive to the relative scale
of logits, and when all logits occupy a narrow numerical range, the loss
operates in a regime of weak gradient response. 

Temperature scaling has previously been explored in classical machine learning, primarily as a post-training calibration technique~\cite{guo2017calibration, mehrtash2020confidence, balanya2024adaptive}. Related temperature-like modifications have also been explored in training ~\cite{das2023powergrad}.
In quantum machine learning, temperature-like parameters
appear only in limited and conceptually distinct contexts, such as contrastive or self-supervised representation learning, where temperature controls similarity sharpness \cite{jaderberg2022quantum}. Separately, temperature arises in quantum Boltzmann and Gibbs-state models as a physical parameter for thermal distributions \cite{amin2018quantum}. None of these approaches addresses the optimization instability arising from the interaction between bounded quantum measurement outputs and softmax-based cross entropy loss sensitivity.

We introduce Quantum Measurement Temperature
(QMT) as a learnable scalar at the quantum measurements and classical loss interface. Unlike prior uses of temperature in quantum machine learning, QMT directly targets the interaction
between bounded measurements and loss sensitivity, restoring gradient strength
without modifying the quantum circuit itself. 

\subsection{Hybrid QNNs on expansion microscopy protein data}

Although hybrid quantum classical neural networks have been applied on standard datasets, their applicability to real fluorescence microscopy imaging remains largely unexplored. In contrast to synthetic or vision benchmark datasets, fluorescence microscopy protein data represent a challenging regime for hybrid quantum models due to strong
intensity fluctuations, photon shot noise, variable point spread
functions, and the intrinsic heterogeneity of protein structures~\cite{truckenbrodt2018x10, truckenbrodt2019practical}. These properties amplify the effects of bounded logits, making protein classification a natural stress test for measurement induced gradient suppression. To test the proposed method, we decided to employ a recent technology, named ONE microscopy~\cite{Shaib2025One}, which delivers highly challenging images of proteins. By combining expansion microscopy~\cite{chen2015expansion, tillberg2016protein}, in which the specimen is physically enlarged, with a fluorescence fluctuation analysis, ONE microscopy increases the resolution of optical microscopes to the nanometer level, thereby enabling the analysis of protein shapes. The images obtained are heterogeneous, indicating proteins in different conformations and positions and therefore serve as an optimal test case for our methodology.

In this work, we evaluate hybrid QNN classifiers directly on fluorescence microscopy
images of protein structures. We show that
measurement-induced logit contraction severely restrict gradient strength on such data, and QMT scaling provides solution to the traianability collapse. The improvement in gradient magnitude,
logit margin, and classification accuracy observed here suggests that hybrid
QNNs, when equipped with appropriate gradient stabilization mechanisms, can provide automated tools for real life applications in fluorescence imaging.
Once this bottleneck is addressed, hybrid QNNs can reliably discriminate between protein classes in fluorescence microscopy data.

\label{sec2}


\section{Method}

\subsection{Hybrid classical-quantum neural network}

We consider a hybrid classical-quantum neural network that combines classical convolutional feature extraction with a variational quantum classifier. The classical
component maps input images to a fixed dimensional feature vector compatible
with the quantum input dimension. These features are encoded as data using Pauli rotation gates in a variational quantum circuit (VQC). The measurement outputs of the VQC acts as logits for multi class classification. The full model is trained using gradient based optimization. Classical
parameters, quantum circuit parameters, and temperature introduced in
this work are optimized jointly.

\subsection{Variational quantum circuit}

The quantum classifier consists of a parameterized quantum circuit acting on
$n$ qubits. Input features are embedded using angle encoding with data
reuploading~\cite{perez2020data}, implemented through single-qubit rotations. Each variational
layer comprises data-dependent rotations, trainable single-qubit rotations, and
entanglement among qubits. Entanglement is applied between neighboring qubits using two-qubit entangling gates in an alternating connectivity scheme. The circuit outputs a vector of
expectation values corresponding to Pauli observables, producing one scalar
output per class.

\subsection{Measurement-induced logit contraction}

Let $z_k$ denote the expectation value of the measurement operator
associated with class $k$. Since these outputs are expectation values of
bounded observables, each logit satisfies the following criterion.
\[
z_k \in [-1, 1].
\]
These bounded logits are passed to a classical loss function for multi-class
classification. The reduction of the effective range of logits due to bounded measurement outputs is termed measurement-induced logit contraction. It arises solely from the quantum measurement process and is independent of circuit depth, entanglement structure, or parameter initialization. As a consequence, the separation between class logits is limited, placing the cross entropy loss in a low-sensitivity regime
with respect to its inputs.

\subsection{Quantum Measurement Temperature (QMT)}

To mitigate the problems associated with measurement-induced logit contraction, we introduce \emph{Quantum Measurement Temperature (QMT)} as a scalar parameter
applied at the measurement interface as shown in Fig.\ref{fig:hybrid}. Given a vector of measurement logits
$z = (z_1, \dots, z_K)$, QMT rescales the logits as
\[
\tilde{z}_k = \frac{z_k}{T},
\]
where $T > 0$ is a learnable parameter.

The rescaled logits $\tilde{z}$ are passed to the softmax function and then used to compute
associated cross entropy loss. All model parameters, including classical weights, quantum circuit parameters,
and the temperature parameter $T$, are optimized jointly using gradient based
methods.

\begin{figure}[H]
    \centering
    \begin{subfigure}[b]{1\textwidth}
        \includegraphics[width=\textwidth]{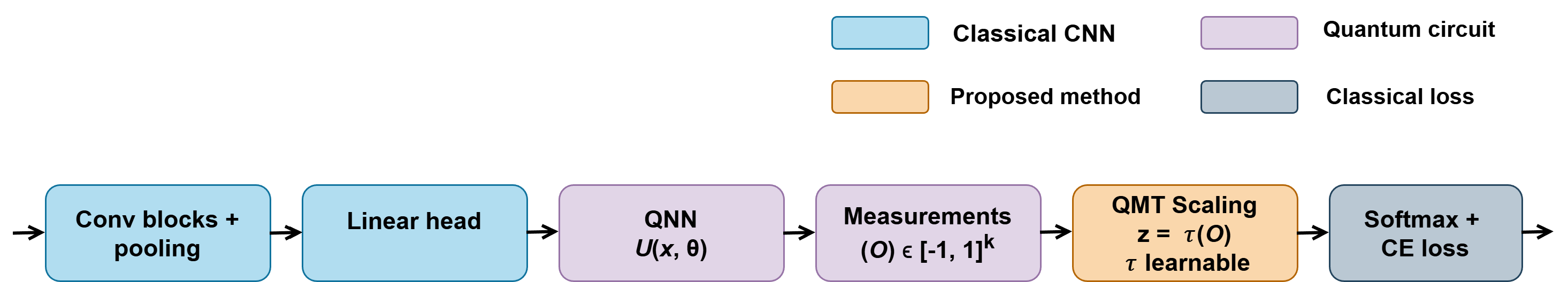}
        \caption{Quantum Measurement Temperature (QMT) to mitigate measurement induced training degradation}
        \label{fig:hybrid}
    \end{subfigure}


    \begin{subfigure}[b]{1\textwidth}
        \includegraphics[width=\textwidth]{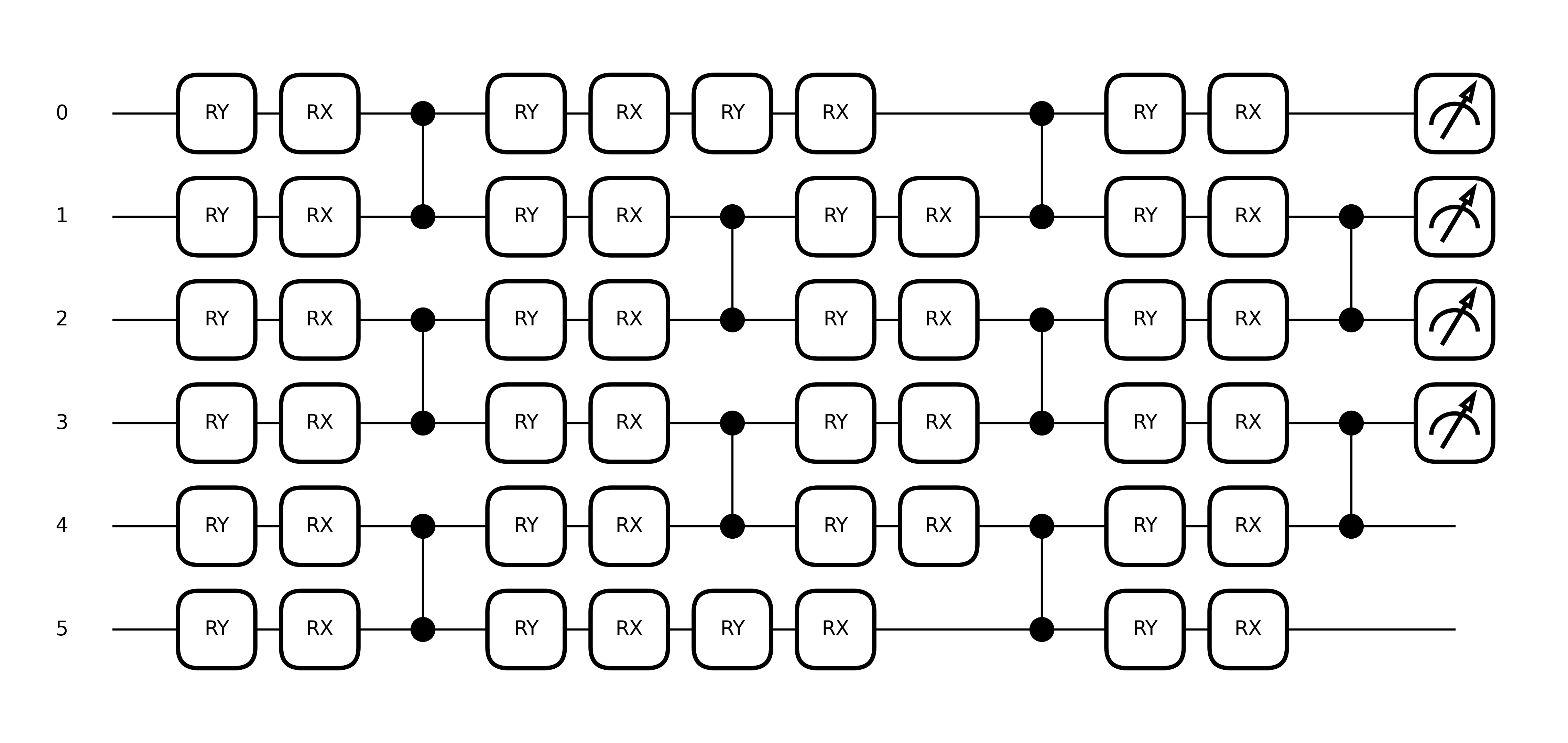}
        \caption{Quantum circuit with rotational gates and entanglements}
        \label{fig:quantum_part}
    \end{subfigure}

    \caption{Hybrid quantum classical neural network with QMT pipeline for improved classification}
    \label{fig:hybrid_vqc}
\end{figure}

\section{Mathematical Analysis:}

Variational Quantum Classifiers (VQC) use the measurement operator's output
$z_i(x,\theta) = \langle \psi(x,\theta)|O_i|\psi(x,\theta)\rangle$
as logits. As standard Pauli observables have
eigenvalues in $[-1,1]$, these logits are intrinsically bounded in this range. In this section, we analyze the consequences of bounded logits on
gradient sensitivity, loss dynamics, trainability, and formally establish
how Quantum Measurement Temperature (QMT) modifies these properties to improve gradient strength and training. The proofs of the following results are provided in the supplementary material.

\subsection{Measurement Induced Logit Compression}

Bounded quantum measurements restrict logits to a narrow numerical range.
This constraint limits both the maximum achievable classification confidence
and the minimum attainable loss, independent of circuit expressivity.
The following result formalizes these limits.

\begin{theorem}[Logit Compression]
\label{thm:compression_main}
Let $z\in\mathbb{R}^K$ satisfy $\|z\|_\infty\le a$. Let
$p_i(z)=\exp(z_i)/\sum_j\exp(z_j)$ denote softmax probabilities.
For a one-hot label $y_c=1$, the maximum probability for the correct class and the minimum loss is given by,
\begin{align}
p_c(z) &\le \frac{e^a}{e^a + (K-1)e^{-a}} \\
\mathcal{L}(z,y) &\ge \log\!\big(1 + (K-1)e^{-2a}\big)
\end{align}
\end{theorem}

This result demonstrates that no matter how expressive the circuit
is, measurement bounds constrain the model's ability to form high margin decisions.
For example, if we consider $K=6$ classes and with standard Pauli measurements, the bounded logits are with $a=1$. Then, the maximum probability possible for the correct class is 
\[
p_c^{\max}
= \frac{e^{a}}{e^{a} + (K-1)e^{-a}}
= \frac{e^{1}}{e^{1} + 5 e^{-1}}
= \frac{e^{2}}{e^{2} + 5}
=\approx \frac{7.39}{7.39 + 5}
\approx 0.597
\]

Thus, even in the \emph{best possible configuration}, the model cannot be
more than $60\%$ confident on a 6 class classification task if logits are restricted
to the interval $[-1,1]$. Due to this constraint, the model is structurally unable to represent highly separated class predictions. This also hampers the loss optimization, as with $a=1$ and $K=6$,
\[
L_{\min}
\ge \log\!\big(1 + 5 e^{-2}\big)
\approx \log\!\big(1 + 5\cdot 0.135\big)
= \log(1.675)
\approx 0.516.
\]

Thus, the loss cannot fall below approximately $0.52$ with standard measurements for this configuration. Due to this logit space compression, the loss landscape
is confined to a regime in which low losses are unreachable. In fact, the more we increase the number of classes, the minimum loss value will be larger and the maximum probability for the correct class will be lesser.

\subsection{QMT for Logit Range Expansion}

We introduce learnable temperature scaling on the logits to improve the measurement induced compression. Rescaling measurement output as $z/T$, allows the model to reach confidence levels and loss values that were previously unattainable under bounded measurements.

\begin{proposition}[Temperature Decompression]
\label{thm:decompression_main}
Let $\tilde z = z/T$ with $\|z\|_\infty \le a$. Then
\begin{align}
p_c^{\max}(a,T)
&= \frac{e^{a/T}}{e^{a/T} + (K-1)e^{-a/T}}, \\
\mathcal{L}_{\min}(a,T)
&= \log\!\big(1 + (K-1)e^{-2a/T}\big),
\end{align}
and as $T\to 0$, $p_c^{\max}(a,T) \to 1$ and
$\mathcal{L}_{\min}(a,T)\to 0$.
\end{proposition}

QMT therefore acts as a geometric decompression operator, expanding logit variation without modifying the quantum circuit itself.

\subsection{Gradient Sensitivity with Respect to Logits}

The next Lemma shows that if the learned temperature T becomes a fractional value, it strengthens logit level gradients by $1/T$ factor.

\begin{lemma}[Logit Gradient Scaling]
\label{thm:loss_main}
For QMT  logits $\tilde z=z/T$,
\[
\frac{\partial\mathcal{L}^{(T)}}{\partial z_i}
= \frac{1}{T}\big(p_i^{(T)} - y_i\big),
\]
where $p^{(T)}=\mathrm{softmax}(z/T)$.
\end{lemma}

With fractional T, the broader logit range gives the optimizer more informative gradients to work with, helping the model learn more reliably despite the bounded nature of quantum measurements.

\subsection{Gradient Variance Improvement}

We are interested in observing parameter gradient variance across input examples as it determines how much useful directional information the optimizer receives during training. When the variance is too small, parameter updates become indistinguishable, and learning stagnates. The following proposition shows that QMT improves parameter gradient variance by $1/T^2$ and recovers from training collapse caused by bounded measurement outputs. 

\begin{proposition}[Parameter Gradient Variance Scaling]
\label{thm:variance_main}
Let $J(\theta)=\partial z/\partial \theta$ denote the quantum Jacobian.
For QMT scaling,
\begin{equation}
\nabla_\theta \mathcal{L}^{(T)}
= \frac{1}{T} J(\theta)^\top (p^{(T)} - y),
\end{equation}
and therefore, when the variance is taken over input data samples,
\begin{equation}
\mathrm{Var}\!\big[\nabla_\theta\mathcal{L}^{(T)}\big]
= \frac{1}{T^2}\,
\mathrm{Var}\!\big[J(\theta)^\top (p^{(T)} - y)\big].
\end{equation}
\end{proposition}

\subsection{Temperature and Loss Geometry}

The following lemma shows that the
Lipschitz constant of the loss scales inversely with temperature. As the Lipschitz constant scales to $L/T$ while using QMT, the loss becomes 
$1/T$ times more responsive to the small logit differences created by bounded quantum measurements.

\begin{lemma}[Lipschitz Scaling]
\label{thm:lipschitz_main}
If $\mathcal{L}$ is $L$ Lipschitz w.r.t.\ logits, and
$\mathcal{L}^{(T)}(z) = \mathcal{L}(z/T)$ then,
\[
|\mathcal{L}^{(T)}(z) - \mathcal{L}^{(T)}(\tilde{z})|
\le \frac{L}{T}\,\|z - \tilde{z}\| \quad
\forall z,\tilde{z}\in\mathbb{R}^K.
\]
\end{lemma}

\subsection{Uniform Bound on Quantum Parameter Gradients}
We show that QMT provides controlled amplification rather than
unbounded gradients. A complete
characterization of temperature therefore requires understanding of how
large parameter space gradients can become with QMT. The following result provides this uniform bound.

\begin{theorem}[Uniform Parameter Gradient Bound]
\label{thm:param_main}
For logits $z_i(\theta)\in[-1,1]$ produced by single parameter Pauli
rotations and bounded observables, the QMT scaled loss satisfies
\[
\left|
\frac{\partial \mathcal{L}^{(T)}}{\partial \theta_k}
\right|
\le \frac{K}{T},
\]
for every parameter $\theta_k$, where $K$ is the number of logits.
\end{theorem}

Pauli generated rotations satisfy the parameter shift rule and because
each shifted expectation remains in $[-1,1]$, the Jacobian entries also follow
$|\partial z_i/\partial \theta_k|\le 1$. QMT scaling gives
$\partial\mathcal{L}^{(T)}/\partial\theta_k
= (1/T)\sum_i (p_i^{(T)}-y_i)(\partial z_i/\partial\theta_k)$.
Using $|p_i^{(T)}-y_i|\le 1$ and the bounded Jacobian yields the result.
\hfill$\square$

\medskip
This establishes that temperature provides a controlled mechanism for
amplifying parameter gradients. When combined with the $1/T^2$
variance scaling from Proposition ~\ref{thm:variance_main}, the result shows
that temperature directly strengthens both the magnitude and the
data  dependent variability of gradients, improving trainability in  Hybrid QNNs.

\section{Experiments}
In our experiments, we use a hybrid Classical-Quantum Neural Network (QNN) architecture
consisting of two classical convolutional layers, a fully connected layer, and
a quantum neural network with a configurable number of qubits and quantum
layers. A representative block diagram of the pipeline is shown in
Fig.~\ref{fig:hybrid}. In this example, the classical frontend uses four
convolutional filters, and the quantum network consists of six qubits and four quantum layers.

The quantum network employs data-encoding gates implemented using $R_Y$
rotations, followed by parameterized $R_X$ rotations and entangling $CZ$ gates,
as illustrated in Fig.~\ref{fig:quantum_part}. The number of measurement
operators is set equal to the number of output classes, e.g., in a
four class classification task, four measurement wires are used. We utilize angle encoding with data re-uploading, and Pauli Z measurements in our experiments. Models are trained for 30 epochs with batch size 64 using the Adam optimizer (learning rate $5\times10^{-3}$) and a cosine annealing learning-rate schedule with $\eta_{\min}=1\times10^{-5}$, unless otherwise specified. We report results as mean $\pm$ standard deviation over 5 independent trials. All experiments were conducted using statevector-based simulations in PennyLane.

\subsection{Datasets}
\begin{figure}[!htbp]
    \centering
    \begin{subfigure}[b]{0.45\textwidth}
        \includegraphics[width=\textwidth]{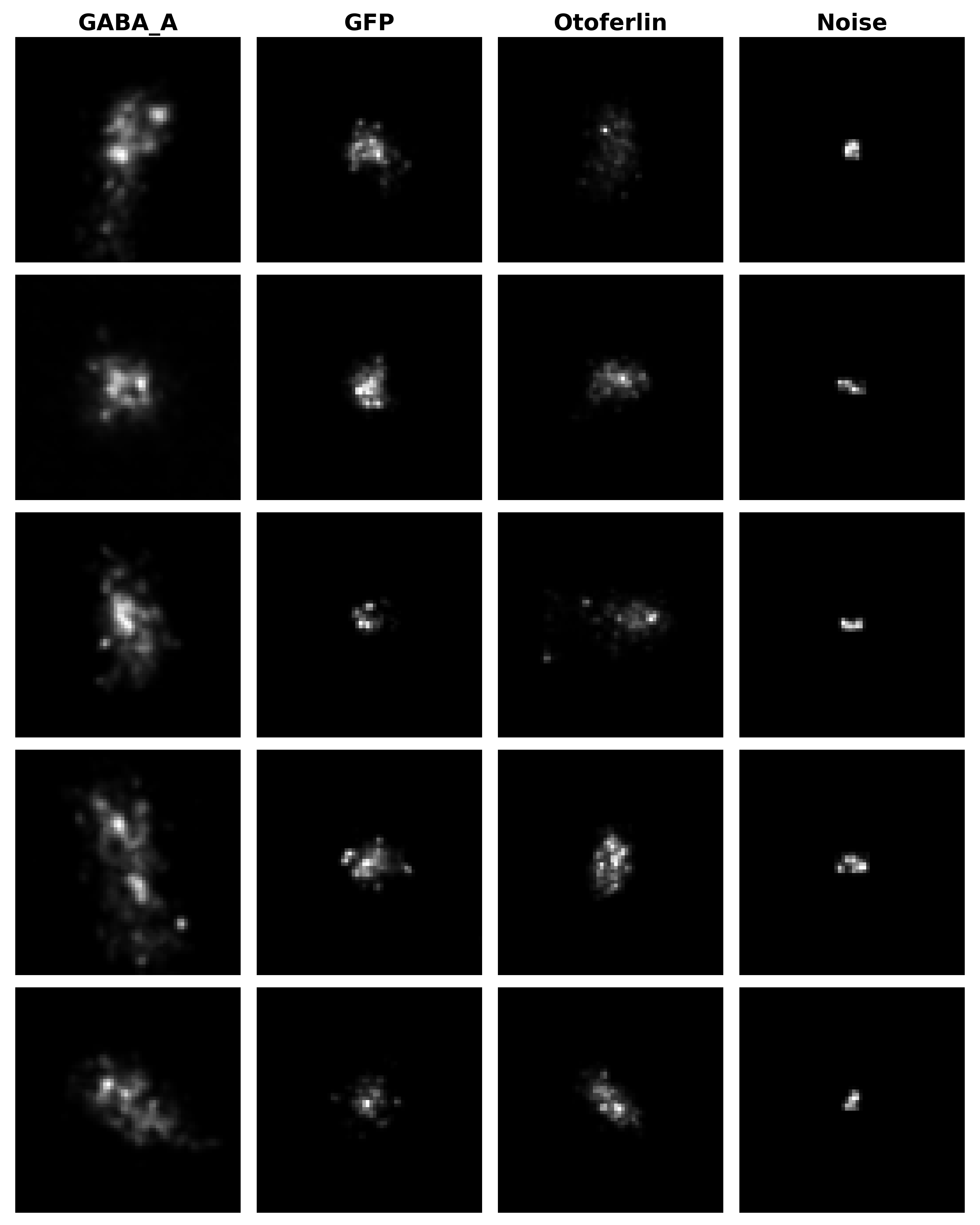}
        \caption{Training dataset}
        \label{fig:ggon_training}
    \end{subfigure}
    \hspace{0.02\textwidth}
    \begin{subfigure}[b]{0.45\textwidth}
        \includegraphics[width=\textwidth]{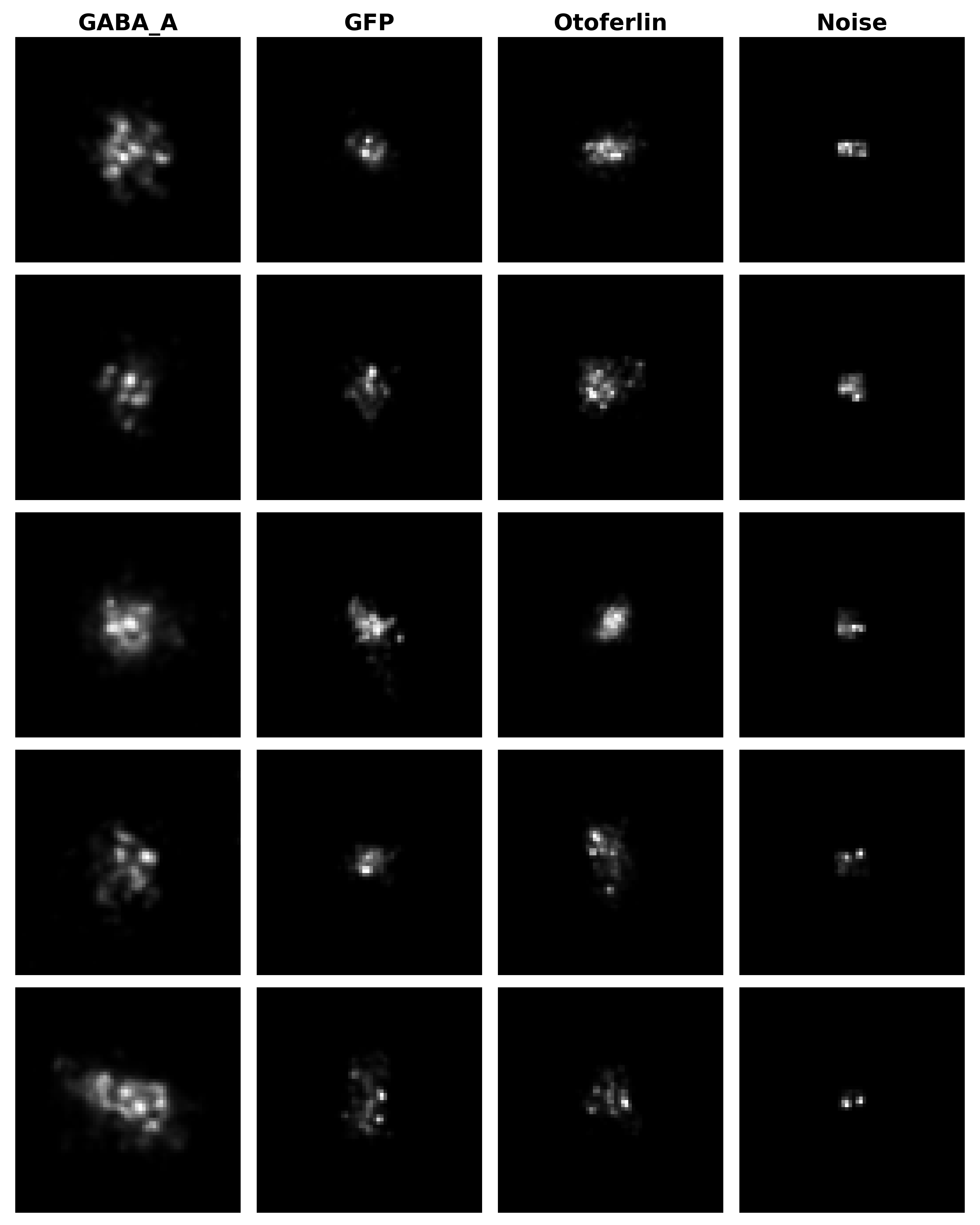}
        \caption{Testing dataset}
        \label{fig:ggon_testing}
    \end{subfigure}
    \caption{Representative training (left) and test (right) samples from the protein dataset. Expanded protein samples and background noise in fluorescence microscopy}
    \label{fig:ggon_data}
\end{figure}

We evaluated the proposed method on two types of datasets: (i) a protein
microscopy dataset and (ii) standard vision benchmark datasets. A $70\%/30\%$ train-test split is used in all 
experiments. The protein dataset is acquired using one-step nanoscale expansion (ONE) microscopy protocol~\cite{Shaib2025One}, which enables nanometer resolution
imaging of protein structures using an optical microscope.
Representative samples from the dataset are shown in
Fig.~\ref{fig:ggon_data}. The dataset contains four classes: three protein
structures (GABA\_A, GFP, and Otoferlin) and a background noise class captured by
the fluorescence microscope, resulting in a four class classification task.
Each class contains 3000 images, and a total of 12,000 images.  The left panel of Fig.~\ref{fig:ggon_data} shows representative
training samples, while the right panel shows representative test samples.

In addition, standard vision benchmark datasets are used to assess the
generality of the proposed method. Specifically, the first six classes of
Fashion MNIST and Overhead MNIST are utilized as multi class classification
benchmarks in our experiments.

\section{Results:}
In this section, we study the results of our experiments in the following order- (i) Observation of Logits or Measurement output range, (ii) Reduction in the Loss value, (iii) Improvement in classification margin, (iv) Overcoming training collapses and establishing training stability, (v) Improvement in parameter gradient strength, (vi) Impact across various QNN architecture and datasets.

\subsection{Observation of Logits or Measurement output range}
We first compare the effective logit range produced by classical networks and
hybrid QNNs. For each epoch, we compute the logit range
$(\max_k z_k - \min_k z_k)$ over the training set and report its evolution during
training, as shown in Fig.~\ref{fig:logits_epochs_all}. The first row corresponds
to the protein dataset (Fig.~\ref{fig:logits_epochs_ggon}) and the second row to
Fashion MNIST (Fig.~\ref{fig:logits_epochs_fmnist}). Within each row, the left
panel shows the classical model, and the right panel shows the hybrid QNN.

\begin{figure}[H]
    \centering
    \begin{subfigure}[b]{1\textwidth}
        \includegraphics[width=\textwidth]{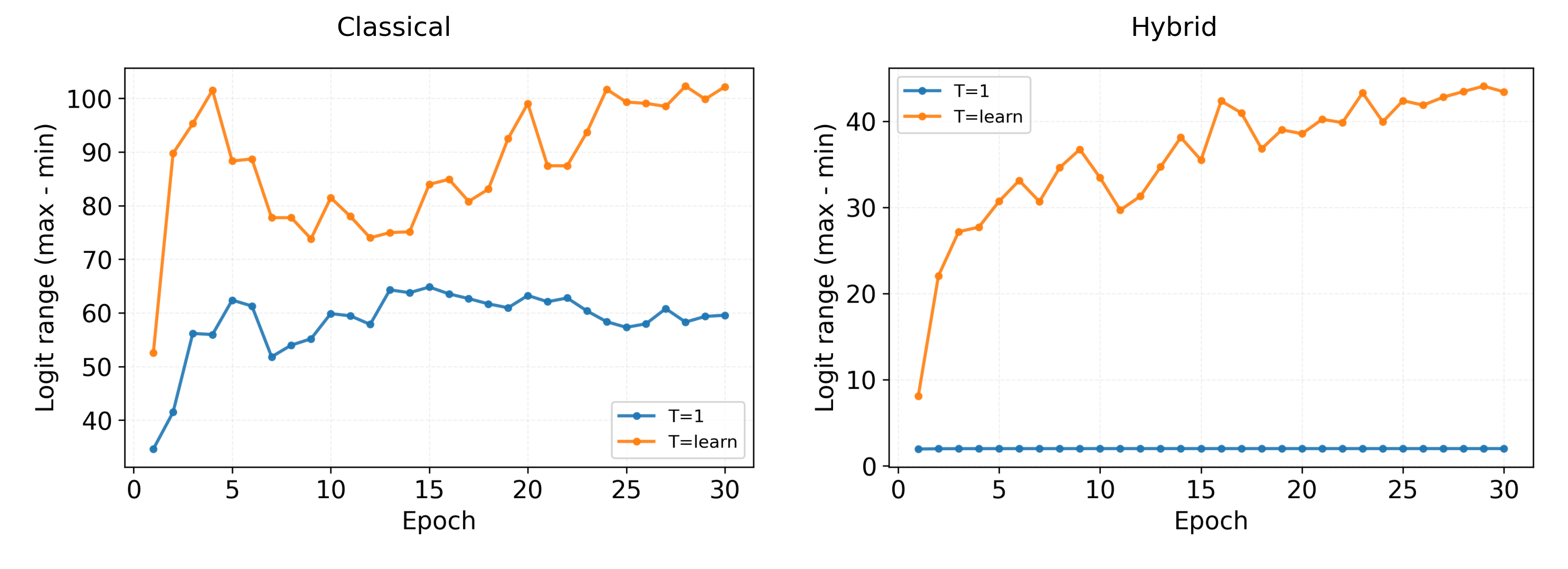}
        \caption{Protein Dataset}
        \label{fig:logits_epochs_ggon}
    \end{subfigure}
    
    
    \begin{subfigure}[b]{1\textwidth}
        \includegraphics[width=\textwidth]{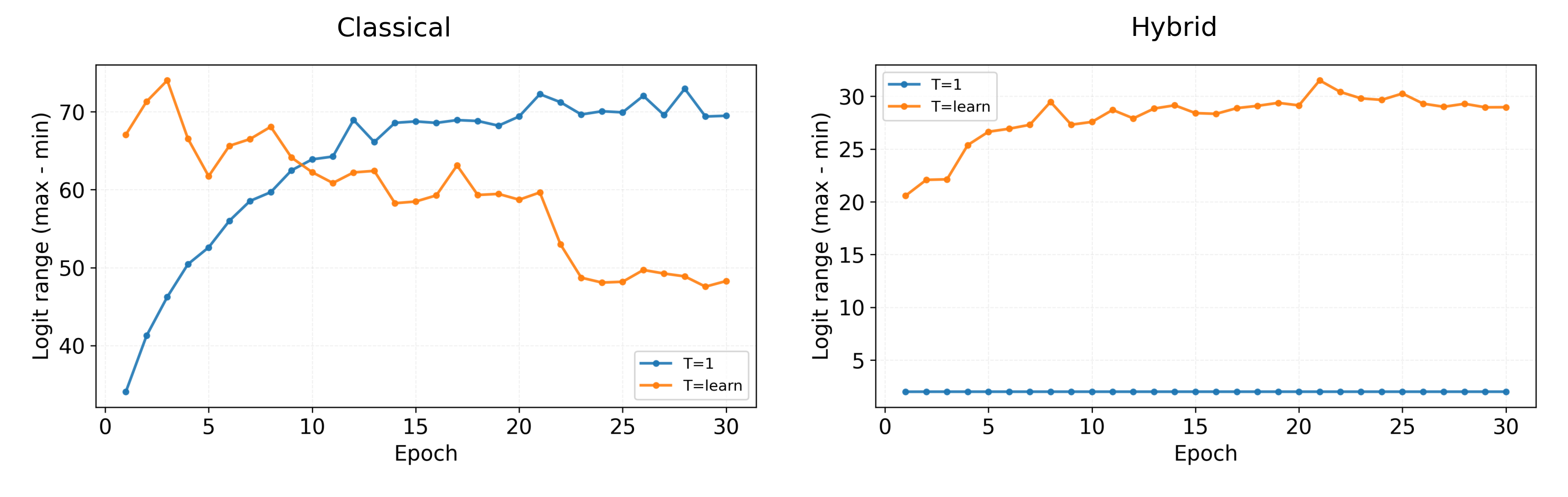}
        \caption{Fashion MNIST}
        \label{fig:logits_epochs_fmnist}
    \end{subfigure}
    \caption{Logit range of Training samples over epochs during training}
    \label{fig:logits_epochs_all} 
    \vspace{-0.4 cm}
\end{figure}

For hybrid QNNs with fixed $T=1$, the measurement outputs satisfy
$z_k \in [-1,1]$, constraining the raw logit range to at most $2$, which is
clearly visible in the right panels of Fig.~\ref{fig:logits_epochs_all}.
Classical models do not exhibit this constraint and therefore explore a larger
logit range during optimization, as shown in the left panels of the same figure.
When QMT is enabled, the rescaled logits $\tilde{z}_k = z_k/T$ expand the
effective logit range for hybrid QNNs, whereas learning $T$ in classical models
alters these ranges in a dataset-dependent manner
without materially affecting optimization as shown in later sections.

Similarly, once the model is trained, the same behavior is observed on the test set. Classical models produce logit ranges well beyond $2$, whereas hybrid QNNs remain confined under fixed $T=1$ and expand their effective logit range only when QMT is learned. The corresponding test-set logit range distribution is reported in Supplementary Fig.~B.12.These observations establish that QMT scaling at the measurement interface expands the effective logit space of hybrid QNNs.

\subsection{Reduction in the Loss value}
We next examine how the expanded effective logit range affects optimization.
Training loss curves for both datasets are shown in Fig.~\ref{fig:loss_epoch_all},
with the protein dataset in Fig.~\ref{fig:loss_epoch_ggon} and Fashion MNIST in
Fig.~\ref{fig:loss_epoch_fmnist}. For classical models, the loss trajectories under fixed $T=1$ and learned $T$
are nearly identical across epochs, indicating that temperature scaling does
not substantially alter optimization behavior when logits are unconstrained.
In contrast, hybrid QNNs with learned QMT converge faster and reach a
significantly lower loss than their fixed-$T$ counterparts, as consistently
observed in the right panels of Fig.~\ref{fig:loss_epoch_all}. These empirical observations are consistent with the theoretical analysis in
Theorem~\ref{thm:compression_main} and Proposition\ref{thm:loss_main}, which predicts improved
loss sensitivity when the effective logit scale is expanded.
\begin{figure}[H]
    \centering
    \begin{subfigure}[b]{1\textwidth}
        \includegraphics[width=\textwidth]{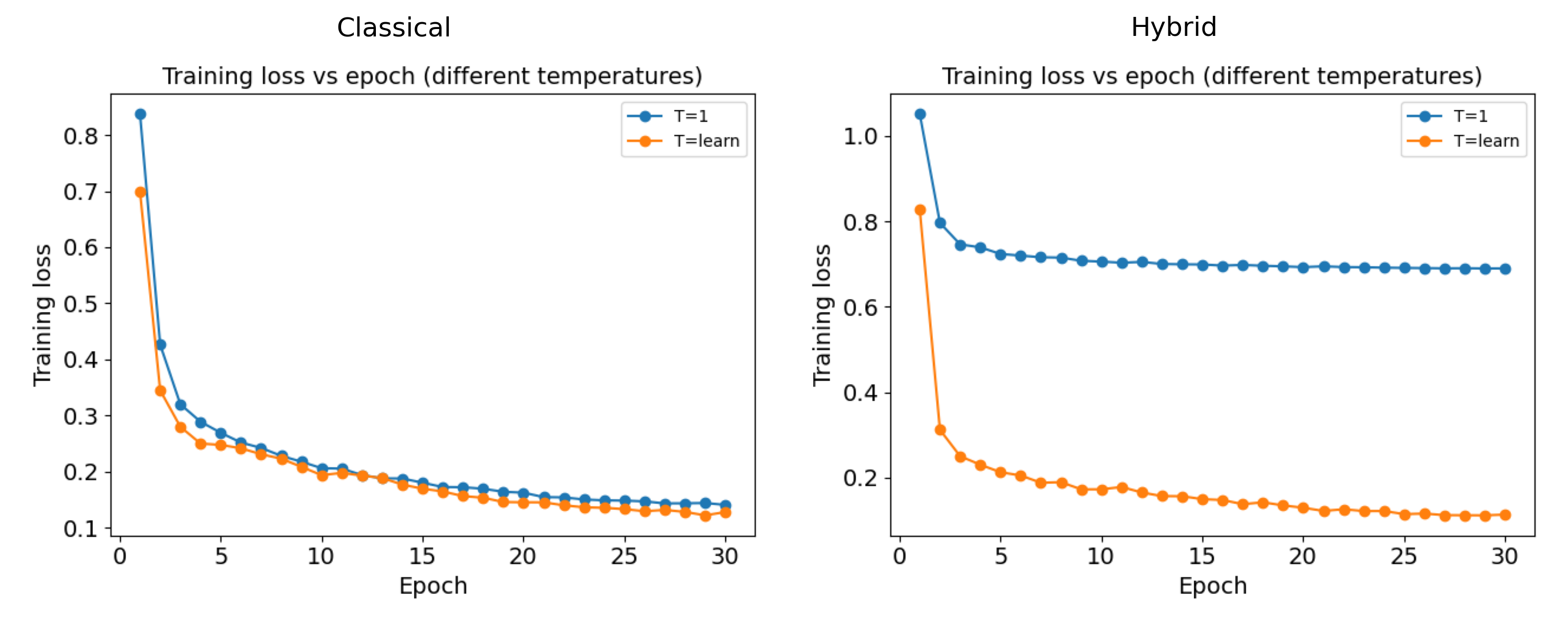}
        \caption{Protein Dataset}
        \label{fig:loss_epoch_ggon}
    \end{subfigure}
    
    
    \begin{subfigure}[b]{1\textwidth}
        \includegraphics[width=\textwidth]{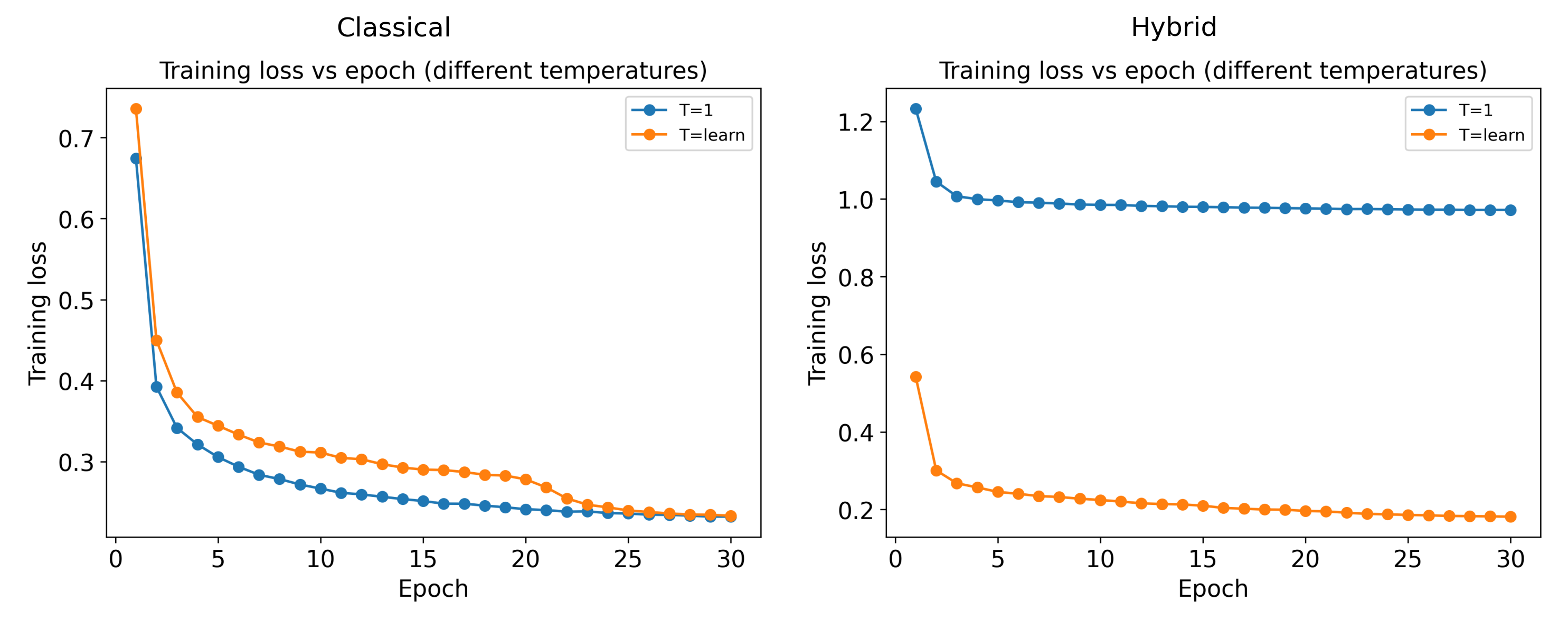}
        \caption{Fashion MNIST}
        \label{fig:loss_epoch_fmnist}
    \end{subfigure}
    \caption{Training loss over epochs for classical and hybrid neural networks}
    \label{fig:loss_epoch_all}
    \vspace{-4 mm}
\end{figure}

\subsection{Improvement in classification margin}
To assess class separation, we evaluate the classification margin
$m(x)= z_{y}(x)-\max_{k\neq y} z_k(x)$ on the test set. The resulting margin
distributions are shown in Fig.~\ref{fig:class_margin_all}, with protein results
in Fig.~\ref{fig:class_margin_ggon} and Fashion-MNIST results in
Fig.~\ref{fig:class_margin_fmnist}. 
For hybrid QNNs, learned QMT shifts the margin distribution toward larger
positive values, indicating improved separation between the true class and the
strongest competing class. The increase in mean margin is modest for classical
models (approximately $1$--$1.5\times$) but substantially larger for hybrid QNNs
(approximately $8$--$10\times$ relative to fixed $T=1$), as evidenced by the
horizontal shift and increased mass at positive margins in
Fig.~\ref{fig:class_margin_all}. These results indicate that QMT improves not only optimization but also the
confidence of hybrid QNN predictions.

\begin{figure}[H]
    \centering
    \begin{subfigure}[b]{1\textwidth}
        \includegraphics[width=\textwidth]{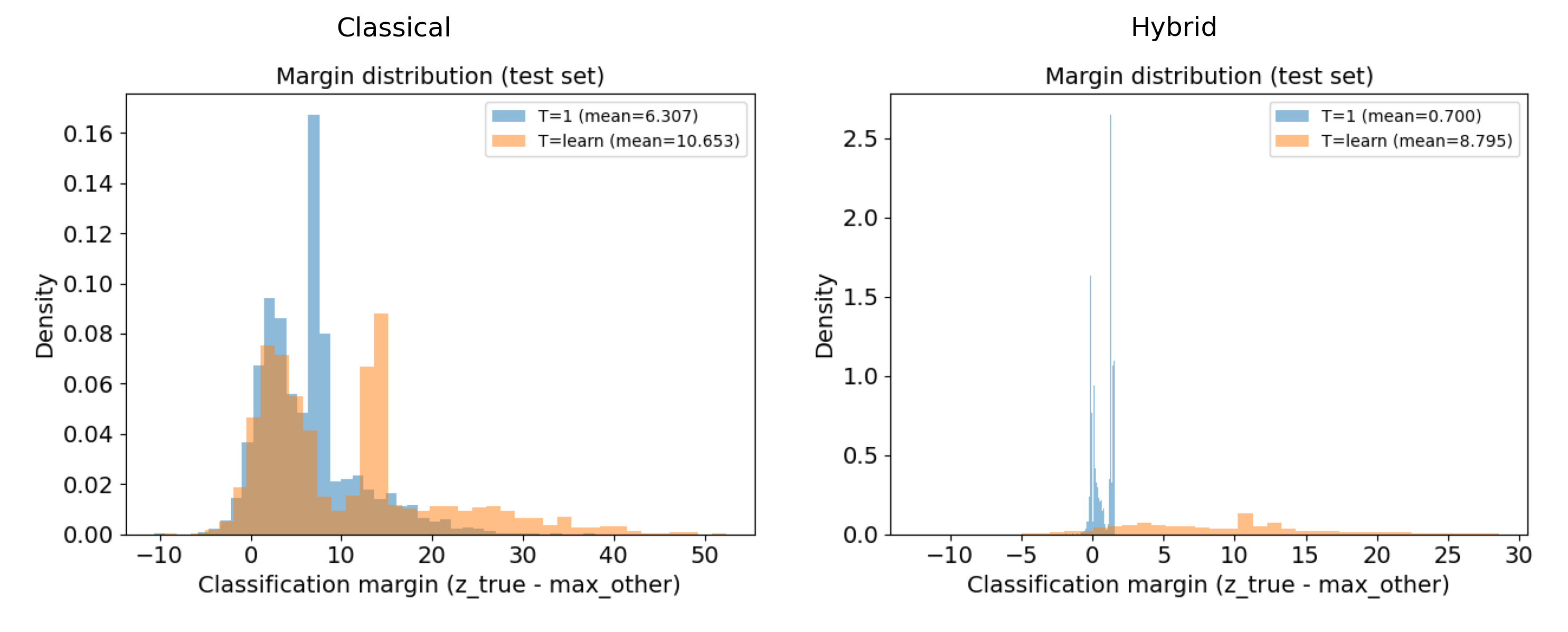}
        \caption{Protein Dataset}
        \label{fig:class_margin_ggon}
    \end{subfigure}
    
    
    \begin{subfigure}[b]{1\textwidth}
        \includegraphics[width=\textwidth]{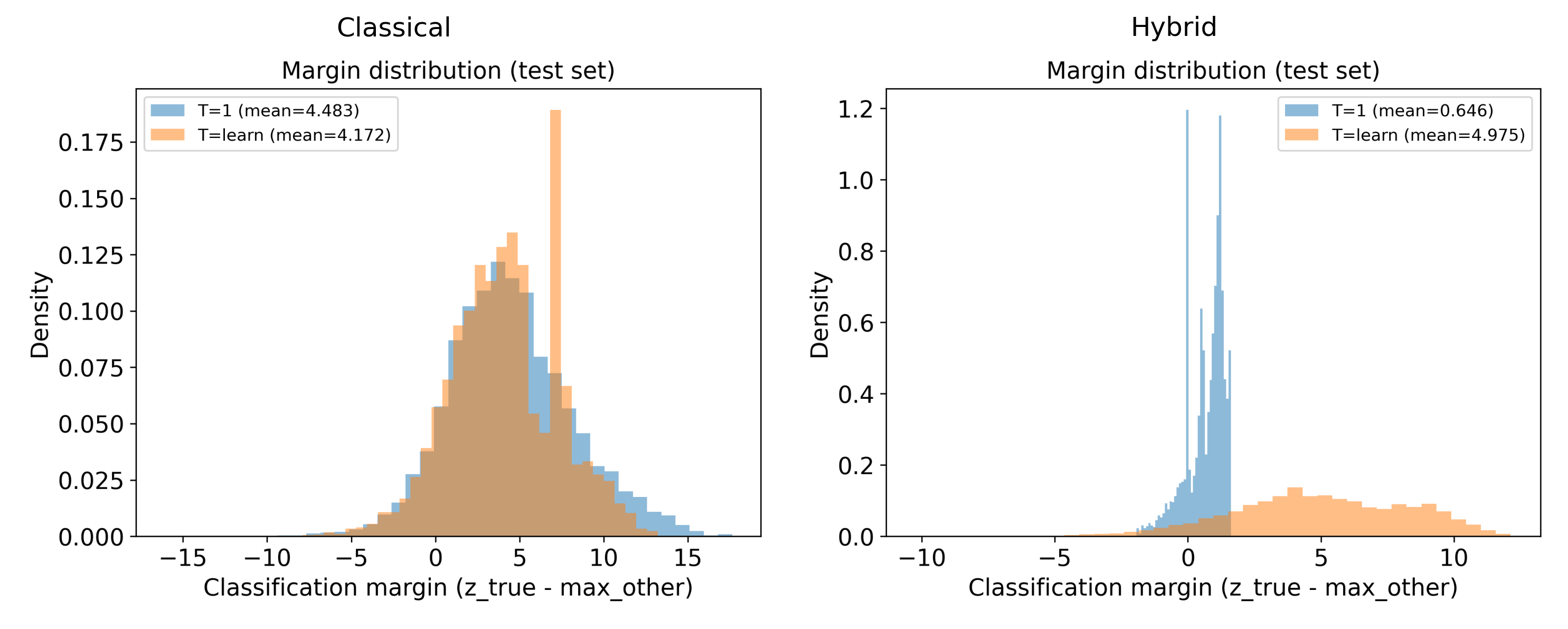}
        \caption{Fashion MNIST}
        \label{fig:class_margin_fmnist}
    \end{subfigure}
    \caption{Test set classification margin distributions with and without QMT for classical NN and hybrid QNN}
    \label{fig:class_margin_all}
    \vspace{-4 mm}
\end{figure}

\subsection{Overcoming Training Collapses and establishing Training Stability}
\begin{figure}[H]
    \centering
    \begin{subfigure}[b]{1\textwidth}
        \includegraphics[width=\textwidth]{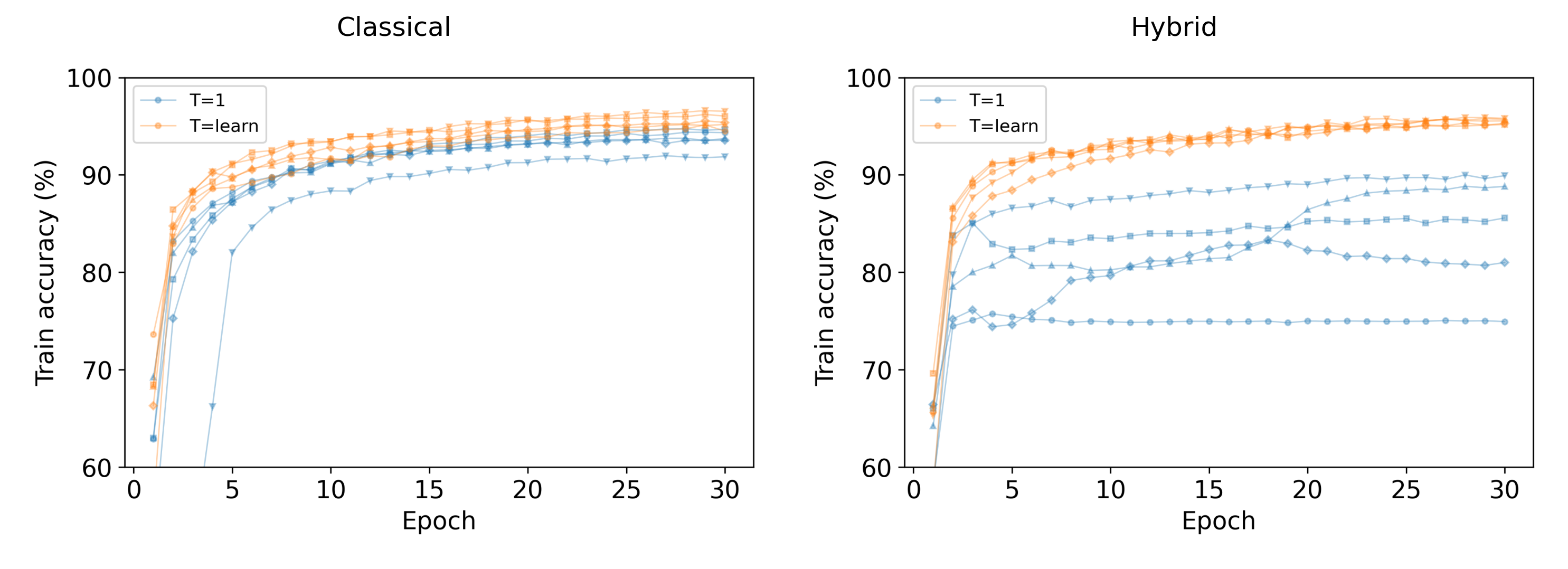}
        \caption{Protein Dataset}
        \label{fig:train_trial_ggon}
    \end{subfigure}
    
    
    \begin{subfigure}[b]{1\textwidth}
        \includegraphics[width=\textwidth]{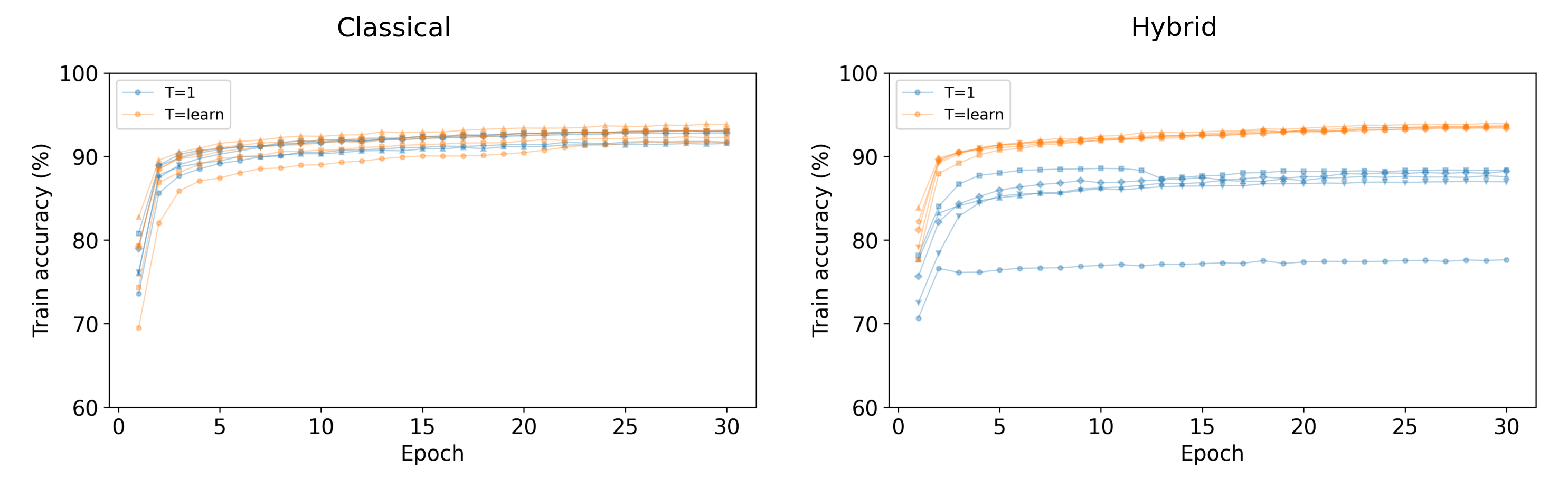}
        \caption{Fashion MNIST}
        \label{fig:train_trial_fmnist}
    \end{subfigure}
    \caption{Training accuracy across five random initializations, illustrating stability effects of QMT for hybrid model}
    \label{fig:train_trial_all} 
\end{figure}
The advantage of introducing learnable QMT scaling is multifold, but the main advantage is that QMT can overcome the training collapses and can provide stability in training along with improvement in the performance. We verify this experimentally using the same model and the same dataset across five trials. Training accuracy trajectories are shown in Fig.~\ref{fig:train_trial_all},
with protein results in Fig.~\ref{fig:train_trial_ggon} and Fashion MNIST
results in Fig.~\ref{fig:train_trial_fmnist}.

For classical models, fixed-$T$ and learned-$T$ runs exhibit similar behavior
across trials, with comparable convergence profiles. In contrast, hybrid QNNs
trained with fixed $T=1$ show pronounced trial-to-trial variability and
occasional training collapse, as seen in the right panels of
Fig.~\ref{fig:train_trial_all}. When QMT is learned, hybrid QNNs consistently
achieve higher training accuracy with substantially reduced variability across
all trials and both datasets.

\subsection{Improvement in parameter gradient strengths}
\begin{figure}[H]
    \centering
    \begin{subfigure}[b]{1\textwidth}
        \includegraphics[width=\textwidth]{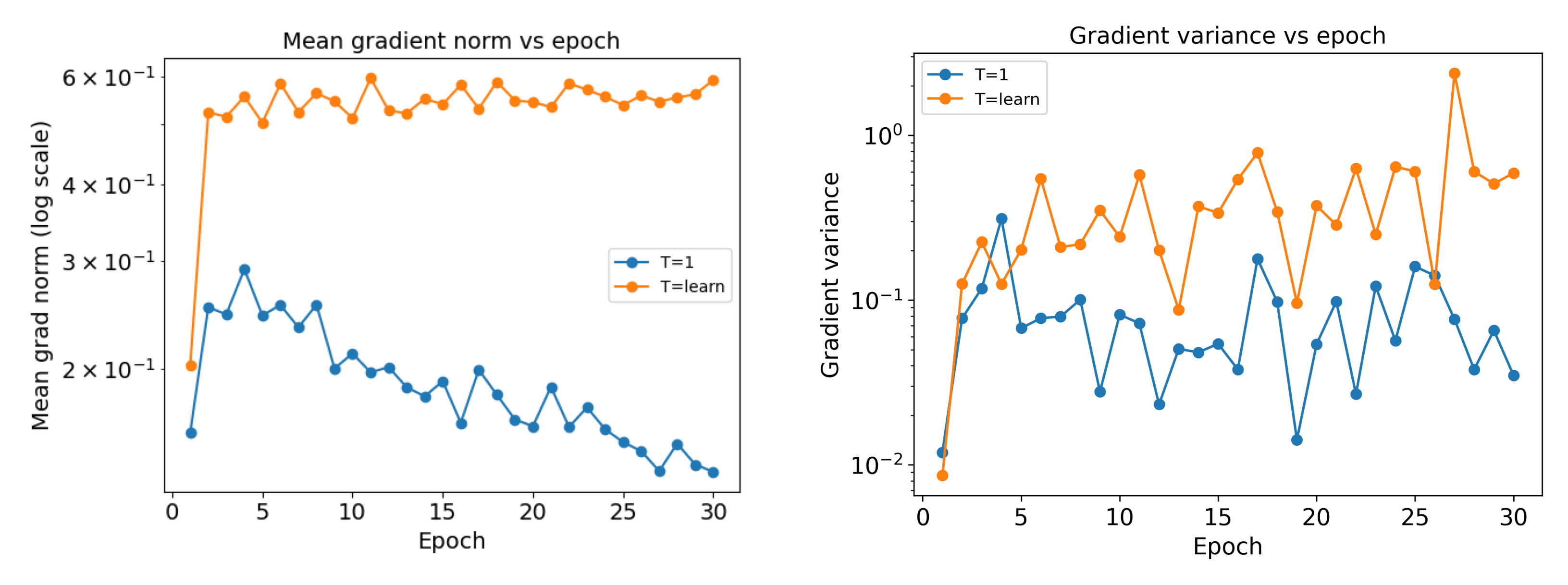}
        \caption{Protein Dataset}
        \label{fig:grad_epoch_ggon}
    \end{subfigure}
    
    
    \begin{subfigure}[b]{1\textwidth}
        \includegraphics[width=\textwidth]{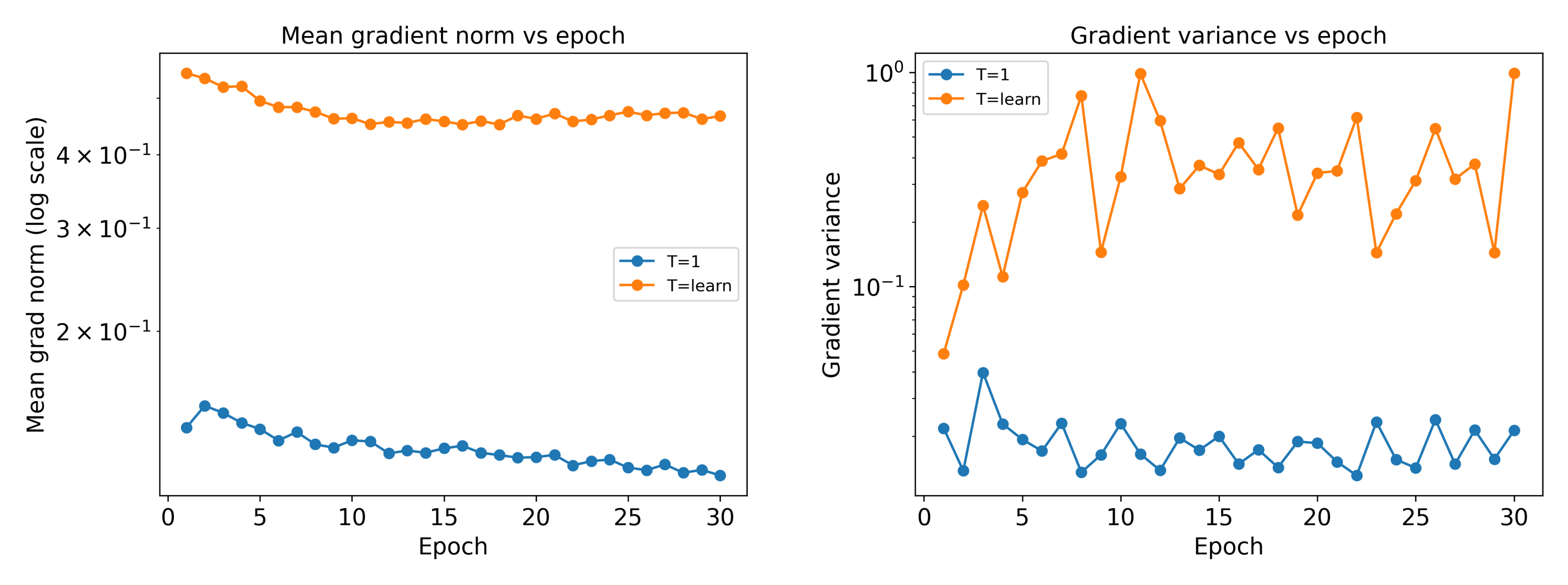}
        \caption{Fashion MNIST}
        \label{fig:grad_epoch_fmnist}
    \end{subfigure}
    \caption{Gradient strength of the trainable parameters over epochs for hybrid QNN used for protein classification task.}
    \label{fig:grad_epoch_all} 
\end{figure}

We have mathematically analyzed how the strength of the parameter gradients gets improved previously in Proposition \ref{thm:variance_main}. Here, we observe the improvement in training accuracy and stability happens as the gradients of the model parameters gets strengthened due to QMT scaling. We observe how gradient norm and gradient variance of the parameters behave during the course of training the model as shown in Fig. \ref{fig:grad_epoch_all}. The mean gradient norm is computed by averaging per parameter gradient norms across parameters and mini-batches, while gradient variance per paramter is computed using all training examples of the first batch in each epoch. High mean gradient norm indicates that the parameter gradient strength gets improved, while higher gradient variance indicates more diverse gradient directions for different parameters. Fig. \ref{fig:grad_epoch_all} shows gradient variance gets improved by more than 10 times while using QMT scaling, which helps the model to learn better. The right plot of Fig. \ref{fig:grad_epoch_all} shows that the gradient variance is too low $(\approx 10^{-2})$ when no QMT is used.
This restoration of gradient variance aligns with the improved convergence and
training stability observed in Fig.~\ref{fig:loss_epoch_all} and
Fig.~\ref{fig:train_trial_all}.

\subsection{Impact while varying QNN architecture and Datasets}

\begin{figure} [!htbp]
    \centering
    \includegraphics[width=\columnwidth]{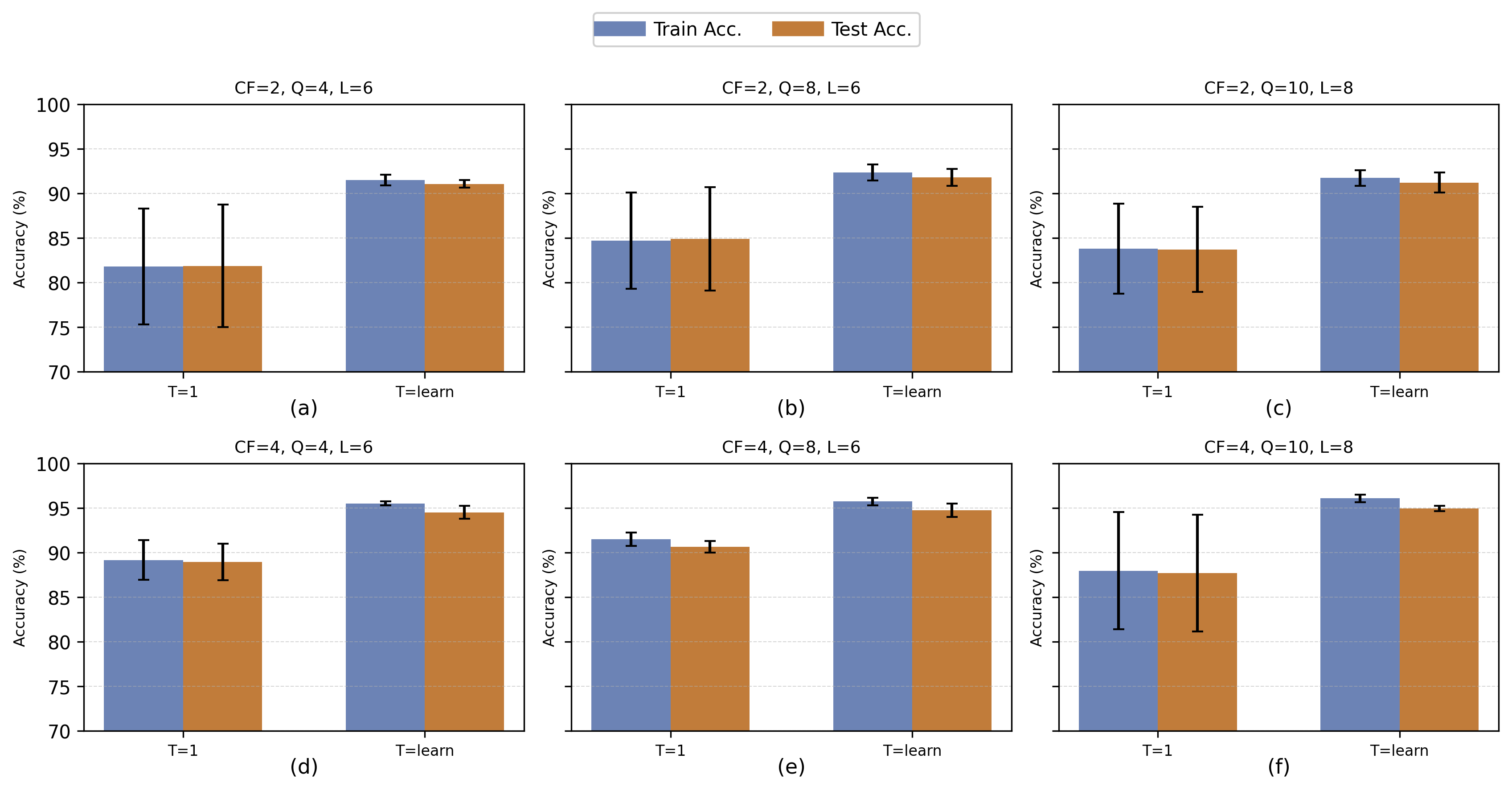}
    \caption{QNN performance comparison for fixed $T=1$ and Learned T while varying architecture configuration. Mean test accuracy and standard deviation on the protein dataset.}
    \label{fig:acc_ggon}
\end{figure}

Figure~\ref{fig:acc_ggon} summarizes the effect of QMT scaling across architectural configurations for protein classification, reporting mean test accuracy and standard deviation over five independent trials while varying the number of convolutional filters (CF), qubits (QB), and quantum layers (QL). Across all configurations, hybrid QNNs with learned temperature consistently outperform their fixed $T=1 $ counterparts in both accuracy and stability.

Without QMT, hybrid QNNs exhibit significant trial-to-trial variability, with the standard deviation of test accuracy exceeding 5.8\% in several settings (Fig.~\ref{fig:acc_ggon}(a,b)). In contrast, learning QMT reduces this variability to below 1\% across trials while simultaneously improving mean accuracy. This stabilization persists for deeper and wider quantum circuits, including configurations with 10 qubits and 8 quantum layers (Fig.~\ref{fig:acc_ggon}(c,f)), indicating that architectural scaling does not resolve training instability without QMT.

Notably, increasing classical capacity does not compensate for the lack of temperature scaling: a hybrid QNN with learned QMT and having two convolutional filters achieves a higher mean accuracy than the same QNN model with four filters without QMT scaling (Fig.~\ref{fig:acc_ggon}(c,f)), highlighting that QMT addresses a distinct optimization bottleneck at the quantum measurement interface.

\begin{table}[H]
\caption{Performance comparison between classical and hybrid models with fixed and learned temperature scaling.}
\label{table:classical_vs_hybrid}
\resizebox{\textwidth}{!}{
\begin{tabular}{@{}ccccccccccccccc@{}}
\toprule
\multirow{2}{*}{Dataset}         & \multicolumn{3}{c}{Arch Config}                              & \multirow{2}{*}{Temp Mode} & \multicolumn{5}{c}{Classical NN}                         & \multicolumn{5}{c}{Hybrid Quantum NN}                            \\ \cmidrule(lr){2-4} \cmidrule(lr){6-10} \cmidrule(lr){11-15} 
                                 & CF                 & QB                 & QL                 &                            & $T_{\text{mean}}$ & $T_{\text{std}}$ & Train Acc (\%) & Test Acc (\%) & Params & $T_{\text{mean}}$ & $T_{\text{std}}$ & Train Acc (\%) & Test Acc (\%)         & Params \\ \midrule
\multirow{8}{*}{Protein Dataset} & \multirow{2}{*}{2} & \multirow{2}{*}{6} & \multirow{2}{*}{4} & Fixed                      & 1.0    & 0.0   & 89.62 ± 2.53   & 89.65 ± 2.13  & 244    & 1.0    & 0.0   & 81.78 ± 5.76   & 81.44 ± 5.76          & 240    \\
                                 &                    &                    &                    & Learned                    & 0.121  & 0.043 & 91.24 ± 1.04   & 90.57 ± 1.07  & 245    & 0.05   & 0.011 & 92.06 ± 0.74   & \textbf{91.98 ± 0.44} & 241    \\
                                 & \multirow{2}{*}{4} & \multirow{2}{*}{6} & \multirow{2}{*}{4} & Fixed                      & 1.0    & 0.0   & 94.14 ± 0.20   & 93.33 ± 0.36  & 654    & 1.0    & 0.0   & 83.60 ± 6.99   & 83.29 ± 6.68          & 650    \\
                                 &                    &                    &                    & Learned                    & 0.095  & 0.023 & 95.07 ± 0.53   & 94.06 ± 0.37  & 655    & 0.04   & 0.008 & 95.73 ± 0.62   & \textbf{94.54 ± 0.61} & 651    \\
                                 & \multirow{2}{*}{4} & \multirow{2}{*}{4} & \multirow{2}{*}{6} & Fixed                      & 1.0    & 0.0   & 94.65 ± 0.62   & 93.82 ± 0.64  & 612    & 1.0    & 0.0   & 89.16 ± 2.23   & 88.94 ± 2.07          & 616    \\
                                 &                    &                    &                    & Learned                    & 0.066  & 0.014 & 94.63 ± 1.48   & 93.60 ± 1.32  & 613    & 0.043  & 0.008 & 95.51 ± 0.23   & \textbf{94.51 ± 0.72} & 617    \\
                                 & \multirow{2}{*}{4} & \multirow{2}{*}{8} & \multirow{2}{*}{6} & Fixed                      & 1.0    & 0.0   & 94.78 ± 0.49   & 93.78 ± 0.34  & 696    & 1.0    & 0.0   & 91.49 ± 0.75   & 90.62 ± 0.65          & 708    \\
                                 &                    &                    &                    & Learned                    & 0.109  & 0.031 & 95.10 ± 0.95   & 94.26 ± 0.89  & 697    & 0.039  & 0.007 & 95.72 ± 0.44   & \textbf{94.74 ± 0.74} & 709    \\ \hline
\multirow{8}{*}{Fashion MNIST}   & \multirow{2}{*}{2} & \multirow{2}{*}{6} & \multirow{2}{*}{4} & Fixed                      & 1.0    & 0.0   & 90.07 ± 0.55   & 89.24 ± 0.35  & 342    & 1.0    & 0.0   & 84.74 ± 0.70   & 83.88 ± 0.50          & 324    \\
                                 &                    &                    &                    & Learned                    & 0.234  & 0.044 & 90.91 ± 0.63   & 90.08 ± 0.71  & 343    & 0.052  & 0.009 & 90.74 ± 0.68   & \textbf{90.18 ± 0.53} & 325    \\
                                 & \multirow{2}{*}{4} & \multirow{2}{*}{6} & \multirow{2}{*}{4} & Fixed                      & 1.0    & 0.0   & 92.82 ± 0.53   & 92.08 ± 0.51  & 708    & 1.0    & 0.0   & 84.95 ± 4.48   & 84.32 ± 4.53          & 690    \\
                                 &                    &                    &                    & Learned                    & 0.317  & 0.056 & 93.14 ± 0.51   & 92.11 ± 0.51  & 709    & 0.054  & 0.007 & 93.41 ± 0.32   & \textbf{92.53 ± 0.56} & 691    \\
                                 & \multirow{2}{*}{4} & \multirow{2}{*}{6} & \multirow{2}{*}{6} & Fixed                      & 1.0    & 0.0   & 92.22 ± 1.45   & 91.10 ± 1.50  & 708    & 1.0    & 0.0   & 86.64 ± 3.23   & 85.90 ± 3.21          & 702    \\
                                 &                    &                    &                    & Learned                    & 0.275  & 0.154 & 92.81 ± 1.09   & 91.95 ± 1.03  & 709    & 0.062  & 0.008 & 93.54 ± 0.21   & \textbf{92.60 ± 0.36} & 703    \\
                                 & \multirow{2}{*}{4} & \multirow{2}{*}{8} & \multirow{2}{*}{6} & Fixed                      & 1.0    & 0.0   & 93.25 ± 0.34   & 92.30 ± 0.29  & 850    & 1.0    & 0.0   & 88.91 ± 2.50   & 88.28 ± 2.72          & 844    \\
                                 &                    &                    &                    & Learned                    & 0.444  & 0.134 & 93.41 ± 0.27   & 92.50 ± 0.37  & 851    & 0.063  & 0.005 & 93.79 ± 0.28   & \textbf{92.77 ± 0.22} & 845    \\ \hline
\end{tabular}}
\end{table}

Quantitative comparisons between classical and hybrid models with comparable parameter counts are reported in Table~\ref{table:classical_vs_hybrid} for both the protein dataset and Fashion MNIST. Learned QMT improves mean test accuracy by 10.5\% and 6.3\%, respectively, for the configuration with CF$=2$, QB$=6$, and QL$=4$, while substantially reducing variance across trials.

The fluorescence microscopy data used here contain proteins at nanometer resolution and exhibit heterogeneous structural patterns. The effect of QMT on improved class separation is illustrated in Fig.~\ref{fig:difficulty_protein}. The margin is defined as the confidence gap between the top-1 and top-2 predicted class probabilities at 
$T=1$. Smaller margins indicate weaker class distinction and increased classification difficulty. When the margin is significantly large (the first two columns), hybrid QNN can predict the correct classes without QMT. However, QMT gains occur predominantly on low-margin samples, where baseline predictions are uncertain due to nearly degenerate logits. Temperature learning increases effective logit separation in these cases, resolving misclassifications without modifying the model architecture as shown in the right sided plots of Fig. \ref{fig:difficulty_protein}.

\begin{figure}[H]

\hspace*{-0.12\linewidth}
\includegraphics[width=1.2\linewidth]
{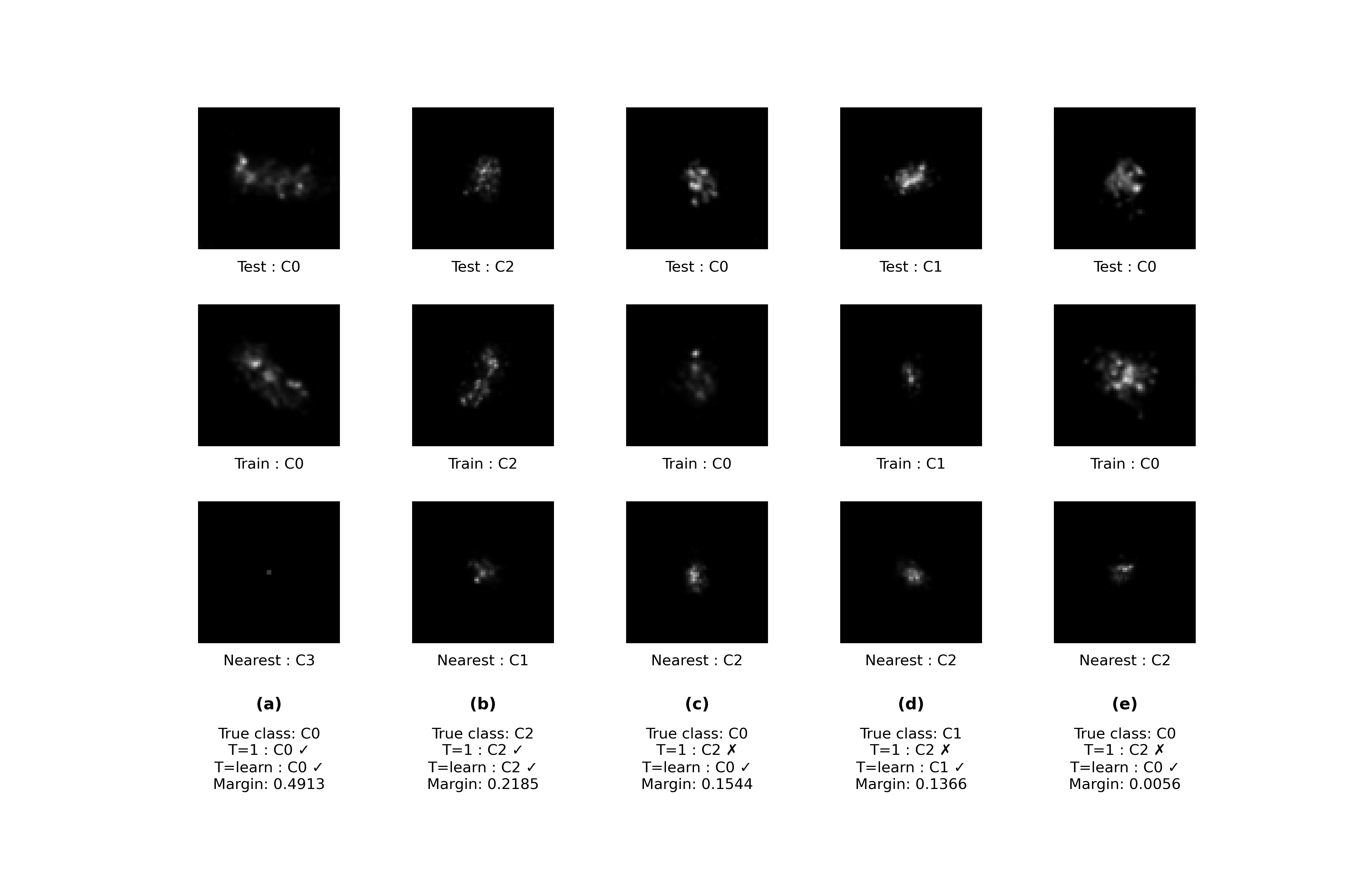}
\caption{Representative hard test samples of proteins at nanometer resolution where QMT changes an incorrect baseline prediction to the correct class. Each column shows a Test sample and two training exemplars: one from the True class (Train) and one from the strongest competing class (Nearest). $T=1$ and 
$T=learn$ refer to the predicted class assignment of the test sample without and with QMT respectively and Margin denotes the confidence gap between the two most likely classes. Learnable QMT enables correct predictions even for low-margin baseline cases.}
\label{fig:difficulty_protein}
\vspace{-4 mm}
\end{figure}

Finally, Table~\ref{table:acc_for_benchmark_datasets} demonstrates that these effects of QMT also generalize to standard benchmarks, including Overhead MNIST and Fashion MNIST, across multiple qubit counts and circuit depths. In particular, configurations with CF$=4$, QB$=8$, and QL$=4$ achieve more than 8.5\% improvement in test accuracy under learned QMT compared to fixed $T=1$, confirming that the benefits of QMT extend beyond the protein classification task.

\begin{table}[H]
\caption{Impact of QMT across benchmark datasets and circuit configurations.}
\label{table:acc_for_benchmark_datasets}
\resizebox{\textwidth}{!}{
\begin{tabular}{@{}ccccccccccc@{}}
\toprule
\multirow{2}{*}{Dataset}        & \multicolumn{3}{c}{Arch Config} & \multicolumn{2}{c}{T = 1}      & \multicolumn{4}{c}{Learned T}                           & \multirow{2}{*}{Params} \\ \cmidrule(lr){2-4} \cmidrule(lr){5-6} \cmidrule(lr){7-10}
                                & CF        & QB       & QL       & Train Acc (\%) & Test Acc (\%) & Train Acc (\%) & Test Acc (\%)         & $T_{\text{mean}}$ & $T_{\text{std}}$ &                         \\ \midrule
\multirow{4}{*}{Overhead MNIST} & 2         & 6        & 4        & 80.06 ± 3.90   & 79.21 ± 3.73  & 91.21 ± 0.34   & \textbf{90.53 ± 0.48} & 0.051  & 0.009 & 325                     \\
                                & 4         & 6        & 4        & 86.16 ± 2.66   & 85.64 ± 2.60  & 93.63 ± 0.24   & \textbf{92.61 ± 0.27} & 0.055  & 0.008 & 691                     \\
                                & 4         & 8        & 4        & 83.86 ± 4.22   & 82.94 ± 4.07  & 93.62 ± 0.16   & \textbf{92.69 ± 0.22} & 0.06   & 0.003 & 829                     \\
                                & 4         & 8        & 8        & 89.76 ± 1.39   & 88.88 ± 1.36  & 93.95 ± 0.30   & \textbf{92.92 ± 0.36} & 0.056  & 0.007 & 861                     \\ \hline
\multirow{4}{*}{Fashion MNIST}  & 2         & 6        & 4        & 84.74 ± 0.70   & 83.88 ± 0.50  & 90.74 ± 0.68   & \textbf{90.18 ± 0.53} & 0.052  & 0.009 & 325                     \\
                                & 4         & 6        & 4        & 84.95 ± 4.48   & 84.32 ± 4.53  & 93.41 ± 0.32   & \textbf{92.53 ± 0.56} & 0.054  & 0.007 & 691                     \\
                                & 4         & 8        & 4        & 84.35 ± 4.41   & 83.80 ± 4.31  & 93.68 ± 0.17   & \textbf{92.64 ± 0.22} & 0.052  & 0.006 & 829                     \\
                                & 4         & 8        & 8        & 88.90 ± 1.54   & 87.80 ± 1.76  & 93.89 ± 0.15   & \textbf{92.80 ± 0.21} & 0.055  & 0.005 & 861                     \\ \hline
\end{tabular}}
\vspace{-4 mm}
\end{table}

We additionally verify that the benefits of QMT are robust to the choice of entanglement structure: replacing CZ with CNOT gates preserves the accuracy and stability gains under learned QMT, while fixed-$T$ hybrid QNNs remain sensitive to the entanglement choice. The corresponding QNN circuit is shown in Supplementary Fig.~B.13. Results in Table~\ref{tab:ggon_entanglements} further show that, without QMT, both entanglement choices show inconsistent convergence and reduced accuracy, whereas learned QMT provides stable and consistently high performance across all tested configurations.
Similar improvements with QMT are observed under variations in both batch size and learning rate as reported in Supplementary Fig.~B.14.

\begin{table}[!htbp]
\caption{Effect of QMT under different entanglement schemes on the protein dataset.}
\label{tab:ggon_entanglements}
\resizebox{\textwidth}{!}{
\begin{tabular}{@{}ccccccccccc@{}}
\toprule
\multirow{2}{*}{Entanglement} & \multicolumn{3}{c}{Arch Config} & \multicolumn{2}{c}{T $=$ 1}      & \multicolumn{4}{c}{Learned T}                           & \multirow{2}{*}{Params} \\ \cmidrule(lr){2-4} \cmidrule(lr){5-6} \cmidrule(lr){7-10}
                              & CF        & QB       & QL       & Train Acc (\%) & Test Acc (\%) & Train Acc (\%) & Test Acc (\%)         & $T_{mean}$ & $T_{std}$ &                         \\ \midrule
\multirow{3}{*}{CNOT}         & 4         & 4        & 6        & 87.51 ± 7.75   & 87.03 ± 7.30  & 95.36 ± 0.65   & \textbf{94.15 ± 0.67} & 0.042  & 0.006 & 617                     \\
                              & 4         & 8        & 6        & 88.54 ± 3.51   & 88.28 ± 2.99  & 95.68 ± 0.72   & \textbf{94.48 ± 0.79} & 0.04   & 0.013 & 709                     \\
                              & 4         & 10       & 8        & 90.96 ± 1.70   & 90.68 ± 1.52  & 95.96 ± 0.86   & \textbf{94.93 ± 0.70} & 0.028  & 0.005 & 775                     \\ \hline
\multirow{3}{*}{CZ}           & 4         & 4        & 6        & 89.16 ± 2.23   & 88.94 ± 2.07  & 95.51 ± 0.23   & \textbf{94.51 ± 0.72} & 0.043  & 0.008 & 617                     \\
                              & 4         & 8        & 6        & 91.49 ± 0.75   & 90.62 ± 0.65  & 95.72 ± 0.44   & \textbf{94.74 ± 0.74} & 0.039  & 0.007 & 709                     \\
                              & 4         & 10       & 8        & 87.96 ± 6.58   & 87.68 ± 6.54  & 96.07 ± 0.41   & \textbf{94.93 ± 0.29} & 0.044  & 0.007 & 775                     \\ \hline
\end{tabular}}
\vspace{-4 mm}
\end{table}

Beyond architectural robustness, we compare the effect of QMT across different optimization strategies. As shown in Fig.~\ref{fig:ggon_optim}, QMT maintains high test accuracy (approximately 93–95\%) across Adam, AdamW, RMSprop, and SGD optimizers. QMT scaling ($T=learn$) not only improves mean test accuracy by ~8–12\% but also reduces performance variability, as shown by a ~5–6× lower standard deviation. This shows that learnable QMT stabilizes training in different optimizers while significantly boosting accuracy.
\begin{figure} [H]
    \centering
    \includegraphics[width=\columnwidth]{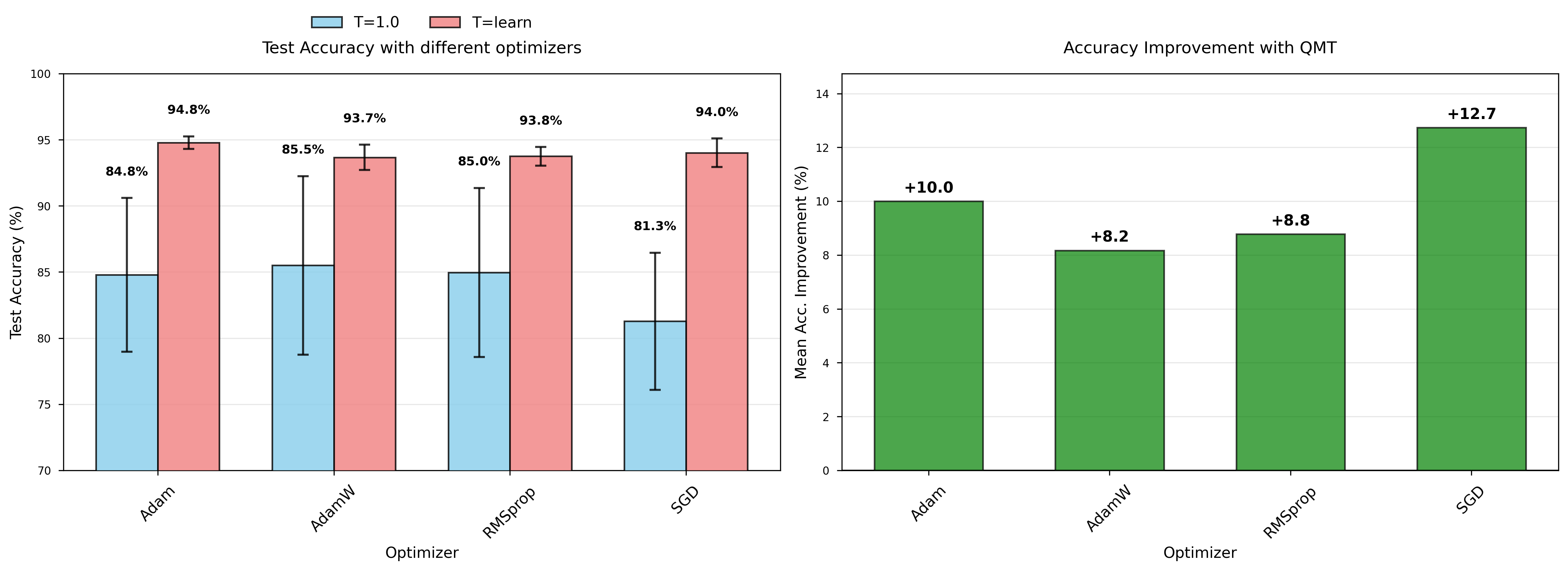}
    \caption{QNN performance comparison for fixed $T=1$ and Learned T while varying optimizers. Mean test accuracy and standard deviation on the protein dataset.}
    \label{fig:ggon_optim}
\end{figure}

To assess how QMT compares with commonly used training strategies for stable training, we evaluate normal training, gradient clipping\cite{pascanu2013difficulty}, and layerwise training \cite{skolik2021layerwise} in both $T=1$ and learned QMT settings. Here, normal training corresponds to simultaneous gradient-based updates of all model parameters. As shown in Fig.~\ref{fig:ggon_stable}, QMT increases mean test accuracy by 9.2\% across normal training, gradient clipping, and layerwise training, with the highest gains (12\%) under gradient clipping. Moreover, QMT reduces the standard deviation by 75\% in normal training. These results indicate that QMT can be applied as a complementary mechanism to improve training reliability in hybrid QNNs.

\begin{figure} [H]
    \centering
    \includegraphics[width=\columnwidth]{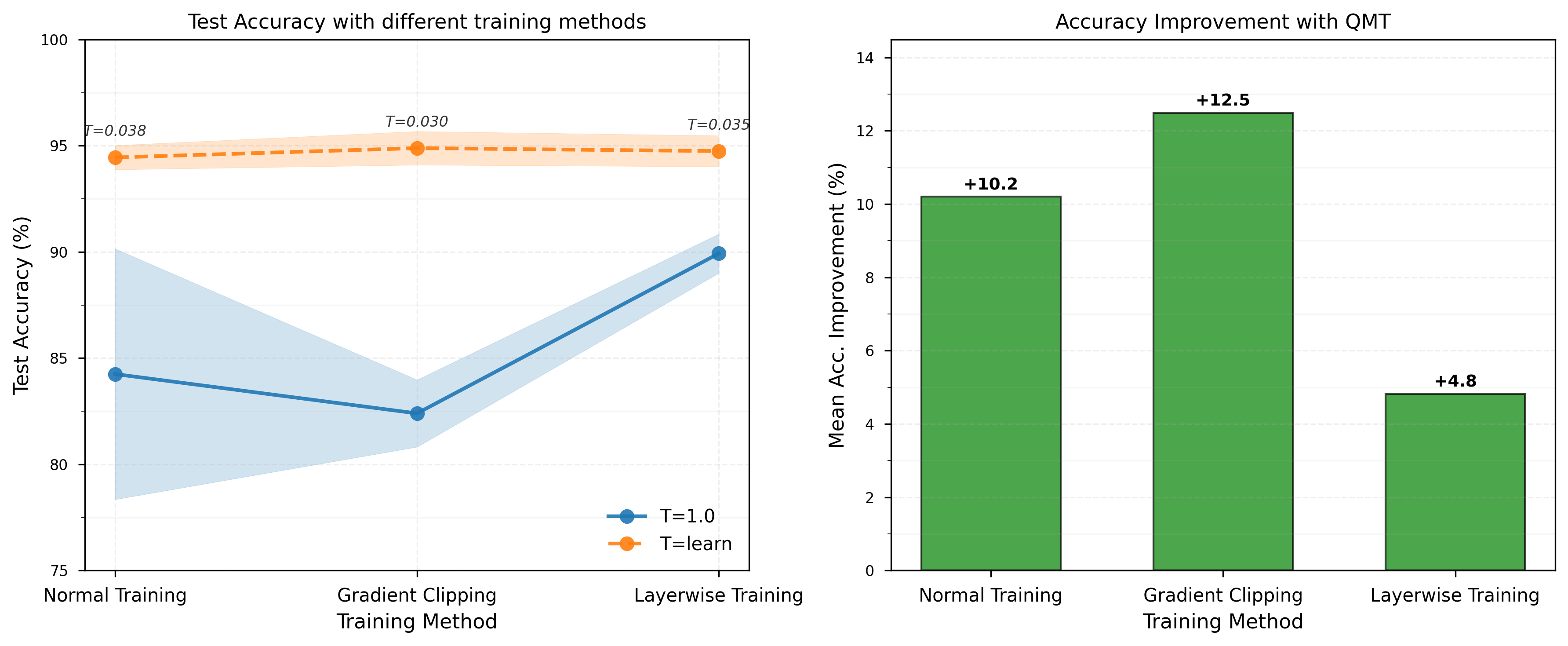}
    \caption{QNN performance comparison for fixed $T=1$ and Learned T with different training stabilization methods. Mean test accuracy and standard deviation on the protein dataset.}
    \label{fig:ggon_stable}
\end{figure}


\label{results}






\section{Conclusions:}
This work identifies and characterizes a structural trainability limitation in hybrid quantum neural network classifiers that arises from the bounded nature of quantum measurement outputs. We show that, when expectation values are used directly as logits, hybrid QNNs can operate in loss regimes where parameter gradient strength is low, which leads to training collapse and instability across trials. 
By introducing Quantum Measurement Temperature (QMT) scaling as an explicit interface between quantum measurements and the loss function, we demonstrate a simple but effective mechanism for restoring parameter gradients, and leading to better optimization. The proposed method does not alter the quantum circuit architecture, depth, or expressivity. Instead, it corrects a mismatch between bounded quantum outputs and loss sensitivity, allowing existing representational capacity to be accessed by gradient based training. 

Beyond classification, the proposed method has immediate implications for other applications using hybrid QNN. Because the limitation originates from bounded quantum outputs, similar optimization issues are expected in regression and other machine learning tasks where loss sensitivity depends on output scale. In such cases, explicit output scaling or loss-aware measurement interfaces can also be important to ensure stable training. Overall, these findings indicate that the progress in hybrid quantum machine learning will depend not only on improved hardware or deeper circuits but also on principled design of the quantum measurement–classical loss interface. Treating measurement outputs as first class elements of the optimization landscape opens new directions for stable and confident training, particularly in real data regimes such as fluorescence microscopy. Addressing this interface explicitly provides a practical path toward deploying hybrid QNNs on realistic, data-intensive problems without increasing quantum resources.

\section*{Declaration of generative AI use}

During the preparation of this manuscript, the authors used an AI-assisted language tool to improve the readability of the text. All scientific content, interpretations, and conclusions were developed by the authors, who reviewed and edited the manuscript and take full responsibility for its content.

\section*{Acknowledgments}
The work was supported by grants from the "Chan Zuckerberg Foundation (iNano)", and "the European Research Council under the European Union’s Horizon 2020 research and innovation program, grant no. 835102". The authors gratefully acknowledge the computing time granted by the Resource Allocation Board and provided on the supercomputer Emmy/Grete at NHR-Nord@Göttingen as part of the NHR infrastructure. The calculations for this research were conducted with computing resources under the project nib00040.

\bibliographystyle{elsarticle-num}
\bibliography{references}




\titleformat{\subsection}
  {\normalfont\bfseries}
  {\thesubsection}{1em}{}
  
\appendix                  
\section*{Supplementary Material}

This supplementary material provides detailed analysis of Quantum Measurement Temperature (QMT), and additional experimental results demonstrating its effects, as referenced in the main manuscript.

\section{Theoretical Proofs}
This section contains the full proofs of the theorems, propositions, and lemmas supporting the analysis of measurement-induced logit compression and QMT scaling presented in the main text.

\subsection{Proof of Theorem 1 (Logit Compression)}
Let $z\in\mathbb{R}^K$ satisfy $\|z\|_\infty\le a$, and
\[
p_i(z) = \frac{e^{z_i}}{\sum_{j=1}^K e^{z_j}}
\]
denote softmax probabilities. For a one-hot label $y_c=1$, the maximum
achievable probability and minimum achievable loss satisfy the following
\begin{align}
p_c(z) &\le \frac{e^a}{e^a + (K-1)e^{-a}}, \\
\mathcal{L}(z,y) &= -\log p_c(z) \ge \log\!\big(1+(K-1)e^{-2a}\big).
\end{align}

\begin{proof}
The softmax probability of the true class $c$ is
\begin{equation}
p_c(z) = \frac{e^{z_c}}{\sum_{j=1}^K e^{z_j}}.
\label{eq:pc-softmax}
\end{equation}
The constraint $\|z\|_\infty \le a$ is equivalent to
\begin{equation}
-a \le z_j \le a \quad \text{for all } j=1,\dots,K.
\label{eq:z-box}
\end{equation}
We first determine the $z$ that maximizes $p_c(z)$ under \eqref{eq:z-box}.
Since the numerator $e^{z_c}$ is increasing in $z_c$ and the denominator
$\sum_j e^{z_j}$ is decreasing in all $z_j$ except $z_c$, the maximum separation of classes is possible when:
\begin{equation}
z_c = a, \qquad z_j = -a \;\;\text{for all } j\neq c.
\label{eq:z-opt}
\end{equation}
Substituting \eqref{eq:z-opt} into \eqref{eq:pc-softmax} gives
\begin{align}
p_c^{\max}(a)
&= \frac{e^a}{e^a + \sum_{j\neq c} e^{-a}} \\
&= \frac{e^a}{e^a + (K-1)e^{-a}}.
\end{align}
This proves the confidence bound of the classifier when quantum measurement output is bounded by $[-a, a]$.

For the loss, we use the fact that the cross entropy for one hot labels is
$\mathcal{L}(z,y) = -\log p_c(z)$. Since $-\log x$ is decreasing in $x$
on $(0,1]$, maximizing $p_c(z)$ minimizes the loss. Hence
\begin{equation}
\mathcal{L}(z,y) \ge -\log p_c^{\max}(a).
\end{equation}
We simplify $-\log p_c^{\max}(a)$:
\begin{align}
-\log p_c^{\max}(a)
&= -\log \left( \frac{e^a}{e^a + (K-1)e^{-a}} \right) \\
&= -\log(e^a) + \log\!\big(e^a + (K-1)e^{-a}\big) \\
&= -a + \log\!\big(e^a + (K-1)e^{-a}\big).
\end{align}
Also,
\begin{align}
\log\!\big(e^a + (K-1)e^{-a}\big)
&= \log\!\big(e^a(1 + (K-1)e^{-2a})\big) \\
&= a + \log\!\big(1 + (K-1)e^{-2a}\big).
\end{align}
Thus,
\begin{align}
-\log p_c^{\max}(a)
&= -a + a + \log\!\big(1 + (K-1)e^{-2a}\big) \\
&= \log\!\big(1 + (K-1)e^{-2a}\big),
\end{align}
which establishes the loss lower bound.
\end{proof}

\subsection{Proof of Proposition 1 (Temperature Decompression)}

\begin{theorem*}
Let $z' = z/T$ with $\|z\|_\infty \le a$, and let softmax
probabilities be $p^{(T)}_i = \mathrm{softmax}(z')_i$.
Then
\begin{align}
p_c^{\max}(a,T)
&= \frac{e^{a/T}}{e^{a/T}+(K-1)e^{-a/T}}, \\
\mathcal{L}_{\min}(a,T)
&= \log\!\big(1+(K-1)e^{-2a/T}\big),
\end{align}
and as $T\downarrow 0$, $p_c^{\max}(a,T)\uparrow 1$ and
$\mathcal{L}_{\min}(a,T)\downarrow 0$.
\end{theorem*}

\begin{proof}
The constraint $\|z\|_\infty \le a$ implies
\begin{equation}
-a \le z_j \le a \quad \Rightarrow \quad -\frac{a}{T}\le z'_j \le \frac{a}{T}.
\end{equation}
Hence the scaled logits $z'$ satisfy
$\|z'\|_\infty \le a/T$. Applying Theorem 1
with $a$ replaced by $a/T$ immediately yields
\begin{align}
p_c^{\max}(a,T) &= \frac{e^{a/T}}{e^{a/T}+(K-1)e^{-a/T}}, \\
\mathcal{L}_{\min}(a,T) &= \log\!\big(1 + (K-1)e^{-2a/T} \big).
\end{align}

To establish the limiting behavior, we have to define $u = a/T$. As $T\downarrow 0$,
we have $u \to \infty$. For the confidence,
\begin{equation}
p_c^{\max}(a,T) = \frac{e^{u}}{e^{u} + (K-1)e^{-u}}
= \frac{e^{2u}}{e^{2u} + (K-1)}.
\end{equation}
As $u\to\infty$, the denominator is dominated by $e^{2u}$, so
$p_c^{\max}(a,T)\to 1$.

For the loss,
\begin{equation}
\mathcal{L}_{\min}(a,T) = \log\!\big(1 + (K-1)e^{-2u}\big).
\end{equation}
Since $e^{-2u}\to 0$ as $u\to\infty$, the argument of the log tends to
$1$, and thus the minimum loss value  $\mathcal{L}_{\min}(a,T)\to 0$.
\end{proof}

\subsection{Proof of Lemma 1 (Logit Gradient Scaling)}

\begin{theorem*}
Let $\mathcal{L}^{(T)}(z,y) = \mathcal{L}(z/T,y)$ be the cross-entropy
loss on temperature scaled logits. Then
\begin{equation}
\frac{\partial \mathcal{L}^{(T)}}{\partial z_i}
= \frac{1}{T}\big(p_i^{(T)} - y_i\big),
\end{equation}
where $p^{(T)} = \mathrm{softmax}(z/T)$.
\end{theorem*}

\begin{proof}
Define $z' = z/T$. By the chain rule,
\begin{equation}
\frac{\partial \mathcal{L}^{(T)}}{\partial z_i}
= \sum_{j=1}^K
\frac{\partial \mathcal{L}(z',y)}{\partial z'_j}
\frac{\partial z'_j}{\partial z_i}.
\end{equation}
We have $z'_j = z_j/T$, hence
$\frac{\partial z'_j}{\partial z_i} = \delta_{ij}/T$, where
$\delta_{ij}$ is the Kronecker delta. Therefore
\begin{align}
\frac{\partial \mathcal{L}^{(T)}}{\partial z_i}
&= \sum_{j=1}^K \frac{\partial \mathcal{L}}{\partial z'_j}
\frac{\delta_{ij}}{T} \\
&= \frac{1}{T} \frac{\partial \mathcal{L}}{\partial z'_i}.
\end{align}
For cross-entropy with softmax,
\begin{equation}
\frac{\partial \mathcal{L}}{\partial z'_i}
= p_i^{(T)} - y_i,
\end{equation}
where $p_i^{(T)} = \exp(z'_i)/\sum_j \exp(z'_j)$. Thus,
\begin{equation}
\frac{\partial \mathcal{L}^{(T)}}{\partial z_i}
= \frac{1}{T}(p_i^{(T)} - y_i).
\end{equation}
\end{proof}

\subsection{Proof of Proposition 2 (Parameter Gradient Variance Scaling)}

\begin{theorem*}
Let $J(\theta)=\partial z/\partial \theta$ be the Jacobian of logits
with respect to parameters. For QMT scaled logits,
\begin{equation}
\nabla_\theta \mathcal{L}^{(T)}
= \frac{1}{T} J(\theta)^\top (p^{(T)} - y),
\end{equation}
and
\begin{equation}
\mathrm{Var}\!\left[\nabla_\theta \mathcal{L}^{(T)}\right]
= \frac{1}{T^2}
\,\mathrm{Var}\!\left[J(\theta)^\top (p^{(T)} - y)\right],
\end{equation}
where the variance is taken over input data samples.
\end{theorem*}

\begin{proof}
From Lemma 1, the gradient with respect to
logits is
\begin{equation}
\nabla_z \mathcal{L}^{(T)} = \frac{1}{T}(p^{(T)} - y).
\end{equation}
Applying the chain rule yields
\begin{equation}
\nabla_\theta \mathcal{L}^{(T)}
= J(\theta)^\top \nabla_z \mathcal{L}^{(T)}
= \frac{1}{T} J(\theta)^\top (p^{(T)} - y).
\end{equation}
For any
scalar $c$, $\mathrm{Var}(cX) = c^2 \mathrm{Var}(X)$, so
\begin{equation}
\mathrm{Var}\!\left[\nabla_\theta \mathcal{L}^{(T)}\right]
= \mathrm{Var}\!\left[\frac{1}{T} J(\theta)^\top (p^{(T)} - y)\right]
= \frac{1}{T^2}\,\mathrm{Var}\!\left[J(\theta)^\top (p^{(T)} - y)\right].
\end{equation}
\end{proof}

\subsection{Proof of Lemma 2 (Lipschitz Scaling)}

\begin{theorem*}
Assume $\mathcal{L}(z,y)$ is $L$-Lipschitz in logits $z$, i.e.
\begin{equation}
|\mathcal{L}(z,y) - \mathcal{L}(\tilde{z},y)|
\le L \|z - \tilde{z}\|
\quad
\forall z,\tilde{z}\in\mathbb{R}^K.
\end{equation}
Then the temperature-scaled loss
$\mathcal{L}^{(T)}(z,y)=\mathcal{L}(z/T,y)$ satisfies
\begin{equation}
|\mathcal{L}^{(T)}(z,y) - \mathcal{L}^{(T)}(\tilde{z},y)|
\le \frac{L}{T} \|z - \tilde{z}\|.
\end{equation}
\end{theorem*}

\begin{proof}
Let $u=z/T$ and $\tilde{u}=\tilde{z}/T$. Then
\begin{align}
|\mathcal{L}^{(T)}(z,y) - \mathcal{L}^{(T)}(\tilde{z},y)|
&= |\mathcal{L}(u,y) - \mathcal{L}(\tilde{u},y)| \\
&\le L \|u - u'\| \\
&= L \left\|\frac{z}{T} - \frac{\tilde{z}}{T}\right\| \\
&= \frac{L}{T} \|z - \tilde{z}\|.
\end{align}
\end{proof}

\subsection{Proof of Theorem 2 (Uniform Parameter Gradient Bound)}

We restate the theorem in a more general form here.

\begin{theorem*}
\label{thm:param-supp}
Let $U(\theta)$ be a variational circuit whose trainable gates are of the
form $e^{-i\theta_k G_k}$ with Hermitian generators $G_k$ of finite
operator norm $\|G_k\|<\infty$. Let $O_i$ be bounded observables with
finite operator norm $\|O_i\|<\infty$, and define logits
\[
z_i(\theta) = \langle \psi_0 | U(\theta)^\dagger O_i U(\theta) | \psi_0 \rangle.
\]
Let $\mathcal{L}^{(T)}(z,y)$ denote the QMT scaled cross entropy
loss. Then there exists a finite constant $C>0$, depending only on
$\{\|G_k\|\}$ and $\{\|O_i\|\}$, such that for all parameters $\theta_k$,
\begin{equation}
\Big|
\frac{\partial \mathcal{L}^{(T)}}{\partial \theta_k}
\Big|
\le \frac{C}{T}.
\end{equation}
In particular, if $G_k$ and $O_i$ are (tensor products of) Pauli operators
with eigenvalues in $\{\pm 1\}$, one can choose $C = K$, where $K$ is the
number of logits.
\end{theorem*}

\begin{proof}
For a logit index $i$ and a parameter $\theta_k$. We know
\[
z_i(\theta)
= \langle \psi_0 | U(\theta)^\dagger O_i U(\theta) | \psi_0 \rangle.
\]
If we isolate the quantum layer containing $k$-th parameterized gate from rest of the circuit, 
\[
U(\theta) = V(\theta_{\neq k})\, e^{-i\theta_k G_k}\, W(\theta_{\neq k}),
\]
where $V$ and $W$ collect all gates not depending on $\theta_k$. Then
\[
z_i(\theta)
= \langle \phi | e^{i\theta_k G_k} \tilde{O}_i e^{-i\theta_k G_k} | \phi \rangle,
\]
where $|\phi\rangle = W(\theta_{\neq k})|\psi_0\rangle$ and
$\tilde{O}_i = V(\theta_{\neq k})^\dagger O_i V(\theta_{\neq k})$.

Differentiating with respect to $\theta_k$ gives
\begin{align}
\frac{\partial z_i}{\partial \theta_k}
&= \frac{\partial}{\partial \theta_k}
\langle \phi | e^{i\theta_k G_k} \tilde{O}_i e^{-i\theta_k G_k} | \phi \rangle \\
&= i\, \langle \phi | e^{i\theta_k G_k}
\big[ G_k, \tilde{O}_i \big]
e^{-i\theta_k G_k} | \phi \rangle,
\end{align}
due to the standard identity
$\frac{\partial}{\partial \theta}
(e^{i\theta G} \tilde{O} e^{-i\theta G})
= i e^{i\theta G}[G,\tilde{O}]e^{-i\theta G}$.

Using the Cauchy–Schwarz inequality in operator
norm, we get
\begin{align}
\bigg|\frac{\partial z_i}{\partial \theta_k}\bigg|
&= \big| \langle \phi | e^{i\theta_k G_k}
\big[ G_k, \tilde{O}_i \big]
e^{-i\theta_k G_k} | \phi \rangle \big| \\
&\le \big\| e^{i\theta_k G_k} [G_k, \tilde{O}_i] e^{-i\theta_k G_k} \big\| \\
&= \big\| [G_k, \tilde{O}_i] \big\|,
\end{align}
as unitary conjugation preserves the operator norm.

The commutator satisfies the standard bound
\[
\|[A,B]\| \le 2\|A\|\|B\|,
\]
which leads to
\begin{equation}
\bigg|\frac{\partial z_i}{\partial \theta_k}\bigg|
\le 2 \|G_k\| \,\|\tilde{O}_i\|
= 2 \|G_k\| \,\|O_i\|.
\end{equation}

From Lemma~1, we know
\[
\frac{\partial \mathcal{L}^{(T)}}{\partial z_i}
= \frac{1}{T} \big(p_i^{(T)} - y_i\big).
\]
So,
\begin{align}
\frac{\partial \mathcal{L}^{(T)}}{\partial \theta_k}
&= \sum_{i=1}^K
\frac{\partial \mathcal{L}^{(T)}}{\partial z_i}
\frac{\partial z_i}{\partial \theta_k} \\
&= \frac{1}{T} \sum_{i=1}^K
\big(p_i^{(T)} - y_i\big)
\frac{\partial z_i}{\partial \theta_k}.
\end{align}
Taking absolute values and using the bounds $|p_i^{(T)} - y_i| \le 1$ for all $i$,
\begin{align}
\bigg|
\frac{\partial \mathcal{L}^{(T)}}{\partial \theta_k}
\bigg|
&\le \frac{1}{T} \sum_{i=1}^K
\big|p_i^{(T)} - y_i\big|
\bigg|\frac{\partial z_i}{\partial \theta_k}\bigg| \\
&\le \frac{1}{T} \sum_{i=1}^K
1 \cdot 2\|G_k\|\|O_i\| \\
&= \frac{2\|G_k\|}{T} \sum_{i=1}^K \|O_i\|.
\end{align}

Defining
\[
C := 2 \max_k \|G_k\| \sum_{i=1}^K \|O_i\|,
\]
we obtain the desired uniform bound
\[
\bigg|
\frac{\partial \mathcal{L}^{(T)}}{\partial \theta_k}
\bigg|
\le \frac{C}{T}
\quad\text{for all } k.
\]

In the special case where both $G_k$ and $O_i$ are tensor products of Pauli
operators with eigenvalues $\pm 1$, we have $\|G_k\| = \|O_i\| = 1$ and
hence $C \le 2K$. Using the parameter-shift rule gives a slightly sharper
constant $C=K$, as stated in the main text.
\end{proof}
The above theorem shows that quantum parameter
gradients are bounded by a temperature dependent term. These proofs highlight the role of QMT scaling in amplifying both gradient magnitude and variance in QNNs.

\section{Supplementary Results}

\subsection{Logit Range and Measurement Output Distributions}

\begin{figure}[H]
    \centering
    \begin{subfigure}[b]{1\textwidth}
        \includegraphics[width=\textwidth]{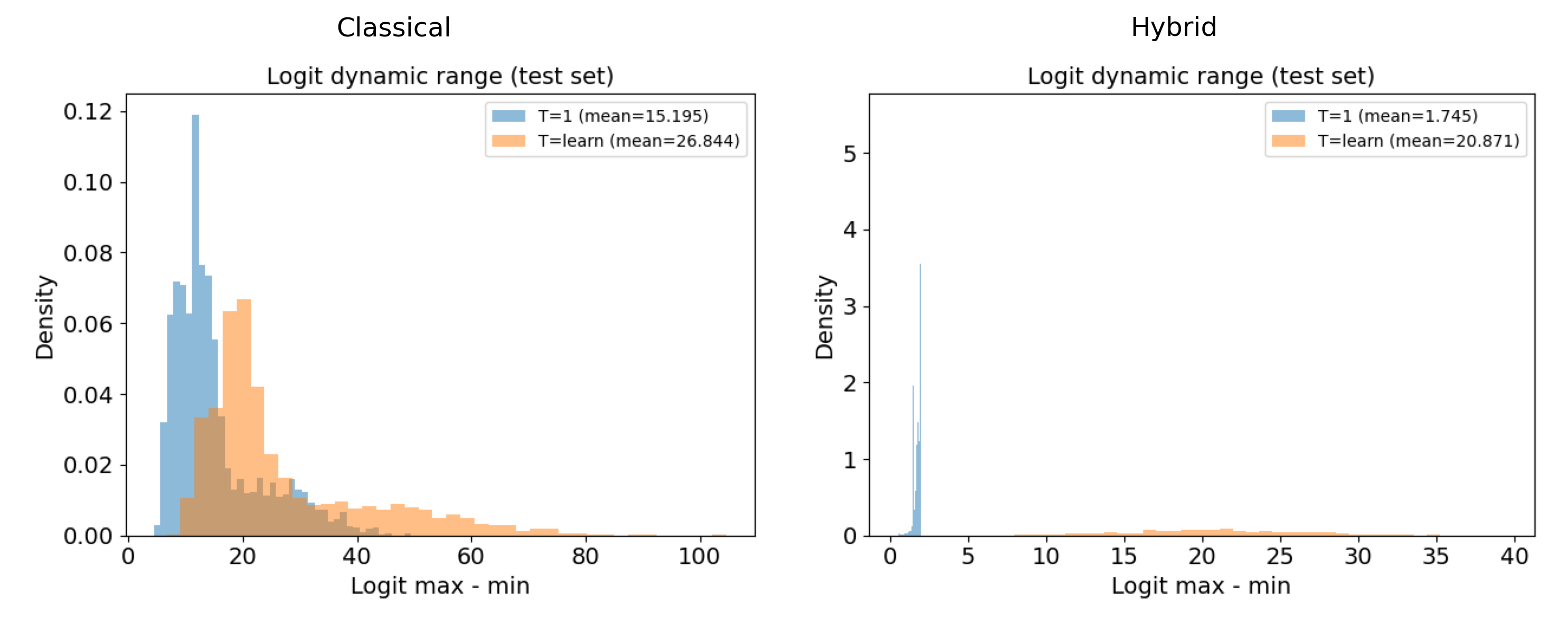}
        \caption{Protein Dataset}
        \label{fig:logits_range_ggon}
    \end{subfigure}
    \begin{subfigure}[b]{1\textwidth}
        \includegraphics[width=\textwidth]{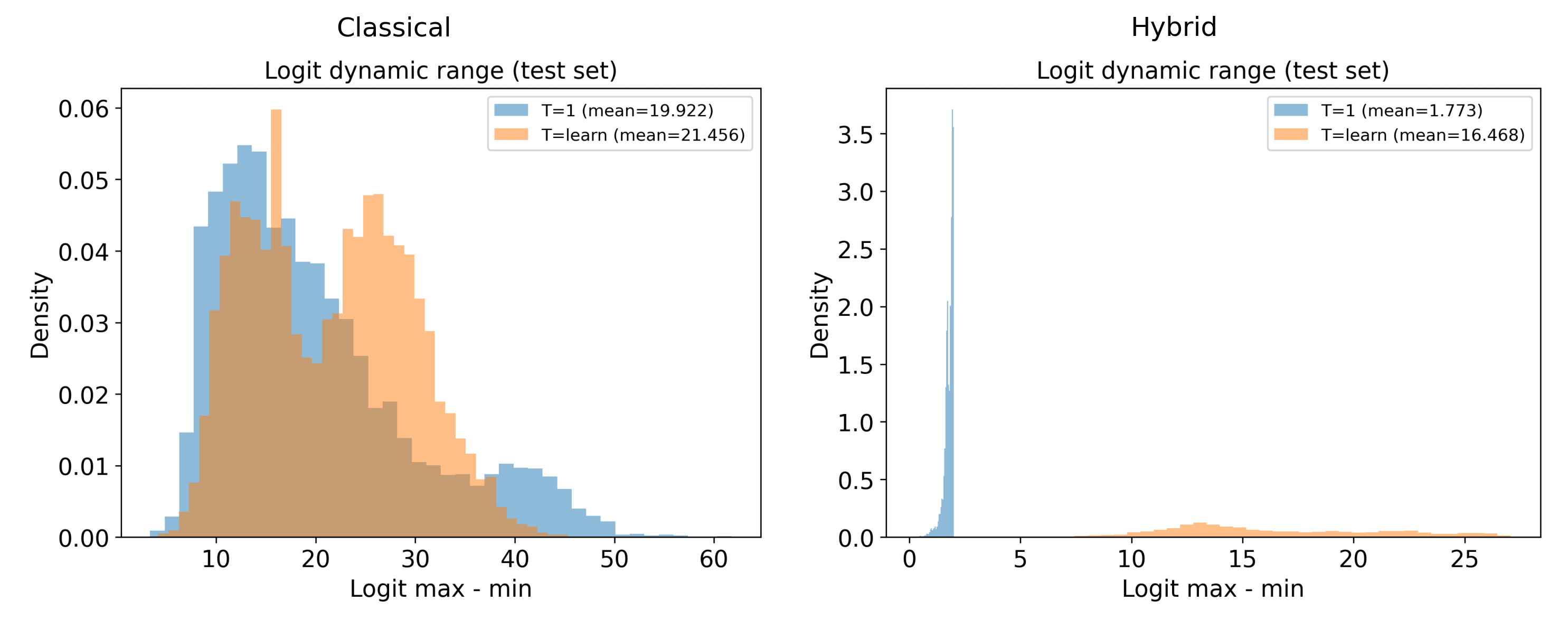}
        \caption{Fashion MNIST}
        \label{fig:logits_range_fmnist}
    \end{subfigure}
    \caption{Distribution of logit ranges on the test set for classical and hybrid models.}
    \label{fig:S1_logit_range}
\end{figure}
The distribution of logit ranges $(\max_k z_k - \min_k z_k)$ computed over all test examples is shown for classical and hybrid models on the protein dataset (Fig. \ref{fig:logits_range_ggon}) and Fashion MNIST(Fig. \ref{fig:logits_range_fmnist}).

Classical models exhibit logit ranges well beyond $2$, while hybrid QNNs remain confined under fixed $T=1$ and expand their effective logit range only when QMT is learned as shown in Fig.\ref{fig:S1_logit_range}.

\subsection{Hybrid QNN with CNOT Entanglement}

\begin{figure}[H]
\centering
\includegraphics[width=\textwidth]{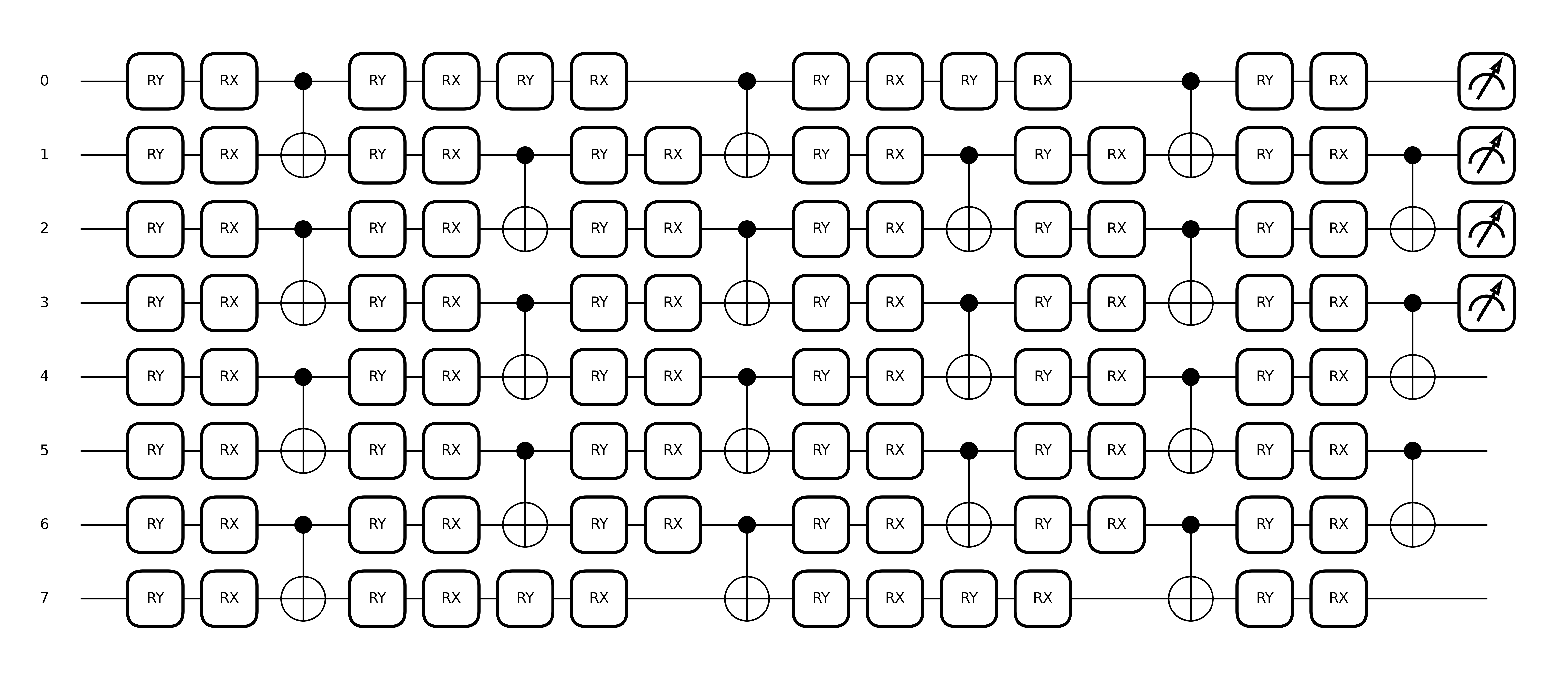}
\caption{Hybrid QNN architecture with CNOT entanglement used for robustness analysis.Hybrid QNN architecture with CNOT entanglement used for robustness analysis. Hybrid QNN with 8 qubits, 6 layers, 4 measurements for protein classification.}
\label{fig:S2_cnot_arch}
\end{figure}

\newpage

\subsection{Effect of QMT with varying batch size and learning rate}
Fig. \ref{fig:ggon_batch_lr} (left) shows that learnable temperature consistently outperforms fixed temperature across all batch sizes, improving test accuracy by 8–15\% and reducing performance variance. In  case of learning rate variation experiment(Fig. \ref{fig:ggon_batch_lr}) , the learnable temperature achieves significantly higher test accuracy (90–95\%) compared to fixed temperature (80–86\%) across all learning rates (0.001–0.1), with optimal performance at lr=0.01.

\begin{figure} [H]
    \centering
    \includegraphics[width=\columnwidth]{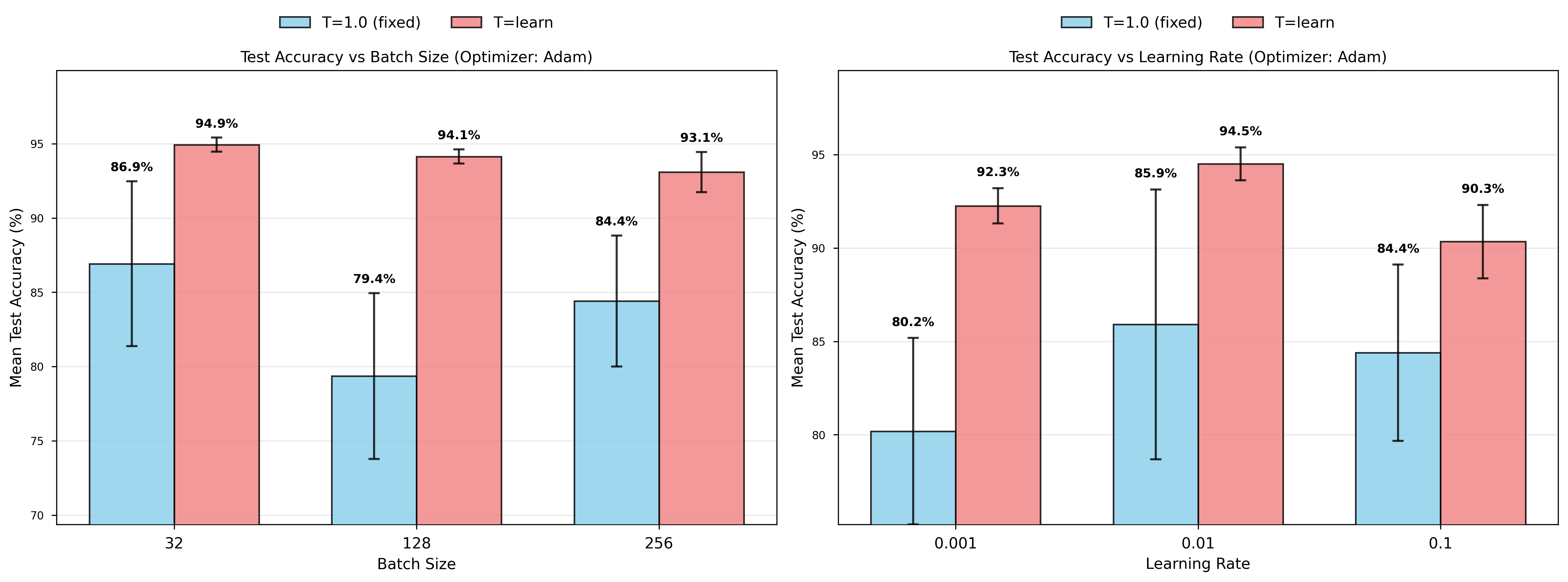}
    \caption{QNN performance on the protein dataset with fixed $T=1$ and learned $T$, while varying batch size (left) and learning rate (right).
    }
    \label{fig:ggon_batch_lr}
\end{figure}


\end{document}